\documentclass{article}
\usepackage{iclr2027_conference,times}
\usepackage[T1]{fontenc}
\usepackage[utf8]{inputenc}
\usepackage{hyperref}
\usepackage{amsmath,amssymb,amsthm,mathtools}
\usepackage{booktabs}
\usepackage{graphicx}
\usepackage{wrapfig}
\usepackage{tikz}
\usetikzlibrary{arrows.meta,positioning,fit}
\usepackage{microtype}
\usepackage{url}
\usepackage{xcolor}
\usepackage{algorithm}
\usepackage{algpseudocode}
\usepackage[capitalise,noabbrev]{cleveref}
\usepackage{pifont}    
\usepackage{colortbl}  
\usepackage{framed}
\definecolor{okgreen}{RGB}{0,135,70}
\definecolor{badred}{RGB}{200,45,45}
\definecolor{oursbg}{RGB}{232,241,252}
\newcommand{\cmark}{\textcolor{okgreen}{\ding{51}}}
\newcommand{\xmark}{\textcolor{badred}{\ding{55}}}
\newcommand{\yes}{\ding{51}}
\newcommand{\no}{\textcolor{black!35}{\ding{55}}}
\newcommand{\good}[1]{\textcolor{okgreen}{\textbf{#1}}}
\newcommand{\pmse}[2]{#1{\scriptsize$\,\pm\,$#2}}

\crefname{assumption}{Assumption}{Assumptions}
\Crefname{assumption}{Assumption}{Assumptions}

\newtheorem{theorem}{Theorem}
\newtheorem{lemma}{Lemma}
\newtheorem{proposition}{Proposition}
\newtheorem{corollary}{Corollary}
\theoremstyle{definition}
\newtheorem{assumption}{Assumption}
\newtheorem{definition}{Definition}
\newtheorem{remark}{Remark}

\newcommand{\E}{\mathbb{E}}
\newcommand{\Prob}{\mathbb{P}}
\newcommand{\cL}{\mathcal{L}}
\newcommand{\kl}{\mathrm{kl}}
\newcommand{\method}{\textsc{Reuse}}

\title{Which Self-Improvements Should We Trust?\\Reliable Self-Improvement When Agents Reuse Their Benchmarks}

\author{
Xiaojing Sun \\
Purdue University
\And
Yuhan Zeng \\
Purdue University
\And
Zihua She \\
Purdue University
\And
Xiao Wang \\
Purdue University
}

\iclrfinalcopy

\begin{document}
\maketitle
\lhead{}

\begin{abstract}
As recursive self-improvement (RSI) rapidly advances, reliable evaluation becomes critical for guiding adaptive search. 
RSI typically relies on finite evaluation resources, such as fixed benchmarks, to determine which modifications are retained and what is proposed next. 
However, when these finite resources are repeatedly reused, new candidates are proposed based on feedback from the same evaluation set, so the search trajectory can adaptively overfit and empirical improvement may not reflect genuine population improvement on the underlying task distribution. 
Some existing methods account for multiple comparisons but assume that candidates are chosen independently of the evaluation set, and therefore do not control this adaptive dependence.
To address this, we propose \method{} (\textbf{R}isk-controlled \textbf{E}valuation \textbf{U}nder \textbf{S}equential \textbf{E}volution), a certified evaluation and promotion framework that allows a fixed evaluation set to support repeated adaptive decisions while providing statistical guarantees. 
Specifically, for a user-specified error level $\alpha$, with probability at least $1-\alpha$, every promoted modification is a genuine population improvement on the underlying task distribution.
\method{} achieves this by strictly limiting the evaluation feedback returned to the search process and accounting for possible promotion histories within the error budget.
We develop a detailed statistical theory for RSI evaluation in this setting, including simultaneous error control, valid lower bounds on cumulative improvement, and a characterization of the fundamental limits of adaptive evaluation reuse.
In live self-improvement experiments, \method{} commits substantially fewer false promotions than evaluation frameworks from current RSI systems and error-controlled baselines, reducing the proportion of promotions that are false from up to 20.7\% to 0\%, while achieving final true population performance comparable to the best baselines.
\end{abstract}



\section{Introduction}
\label{sec:intro}

Recent progress in recursive self-improvement (RSI) has significantly advanced the capabilities of self-improving AI systems.
For instance, agents now rewrite their own tools \citep{dgm2025,sica2025}, evolve production-ready algorithms \citep{alphaevolve2025}, and automate complex workflows like model alignment and post-training \citep{anthropic2026aar,posttrainbench2026}.
As agents become more capable at proposing and executing changes \citep{zelikman2024stop,hu2025adas,lu2024aiscientist,huxley2025}, evaluation becomes increasingly critical \citep{manheim2018goodhart,skalse2022defining,gao2023scaling,guo2026seal}. 
The evaluation feedback dictates which candidate modifications are retained, expanded, or used to generate future proposals. 
Ideally, this feedback would favor only changes that improve performance on the underlying task distribution \citep{hastie2009elements,recht2019imagenet}.
In practice, current RSI systems employ diverse search strategies but typically obtain this feedback by repeatedly querying a finite set of fixed benchmarks \citep{wang2026rethinking}.

This creates a fundamental statistical challenge. 
When RSI systems rely on a finite set of benchmarks to guide adaptive search, the search trajectory becomes increasingly tuned to the evaluation signal \citep{pace2026}, and empirical improvement may reflect adaptation to these benchmark tasks rather than genuine improvement \citep{wang2026rethinking, anthropic2026aar}. 
False promotions in this setting arise from three main sources.
\begin{itemize}
    \item \textbf{Within-round selection.} Selecting the best of several candidates within a round favors candidates whose advantage on the evaluation set is overstated by chance \citep{cawley2010overfitting,efron2011tweedie,siren2026}.
    \item \textbf{Repeated testing.} Testing across many rounds increases the probability that at least one test makes an error \citep{dunn1961multiple,holm1979simple}.
    \item \textbf{Adaptive dependence.} Later candidates are adapted to earlier evaluation outcomes \citep{dwork2015reusable,hardt2014preventing,blum2015ladder}.
\end{itemize}
The first two would arise even if every round used fresh evaluation data, and standard multiple comparison corrections handle them by accounting for the number of tests \citep{dunn1961multiple,holm1979simple,bretz2009graphical,tian2021online}. 
The third source arises from benchmark reuse. 
When later candidates are generated based on earlier outcomes on the evaluation set and then tested on that same set, the candidates themselves become functions of the evaluation set. 
This adaptive dependence violates the assumption of standard statistical tests that candidates are chosen independently of the evaluation set \citep{kriegeskorte2009circular,berk2013valid,dwork2015generalization}. 
Consequently, a candidate may be promoted because it overfits the specific evaluation set, rather than because it genuinely improves on the underlying task distribution.
Recent work attempts to bring statistical error control into self-improvement loops and iterative evaluation \citep{pace2026, sgm2025, sea2026, guo2026seal, bertran2026fits}. However, none of these methods account for all three sources of false promotion simultaneously. In particular, their statistical guarantees rely on each candidate being independent of the data on which it is evaluated, and restricting feedback alone does not account for the information revealed by the promotion decisions themselves. Because they do not address the adaptive dependence created by reusing the same evaluation set, they cannot guarantee that every promoted modification is a genuine population improvement.

\begin{figure*}[h]
    \centering
    \includegraphics[width=0.92\textwidth]{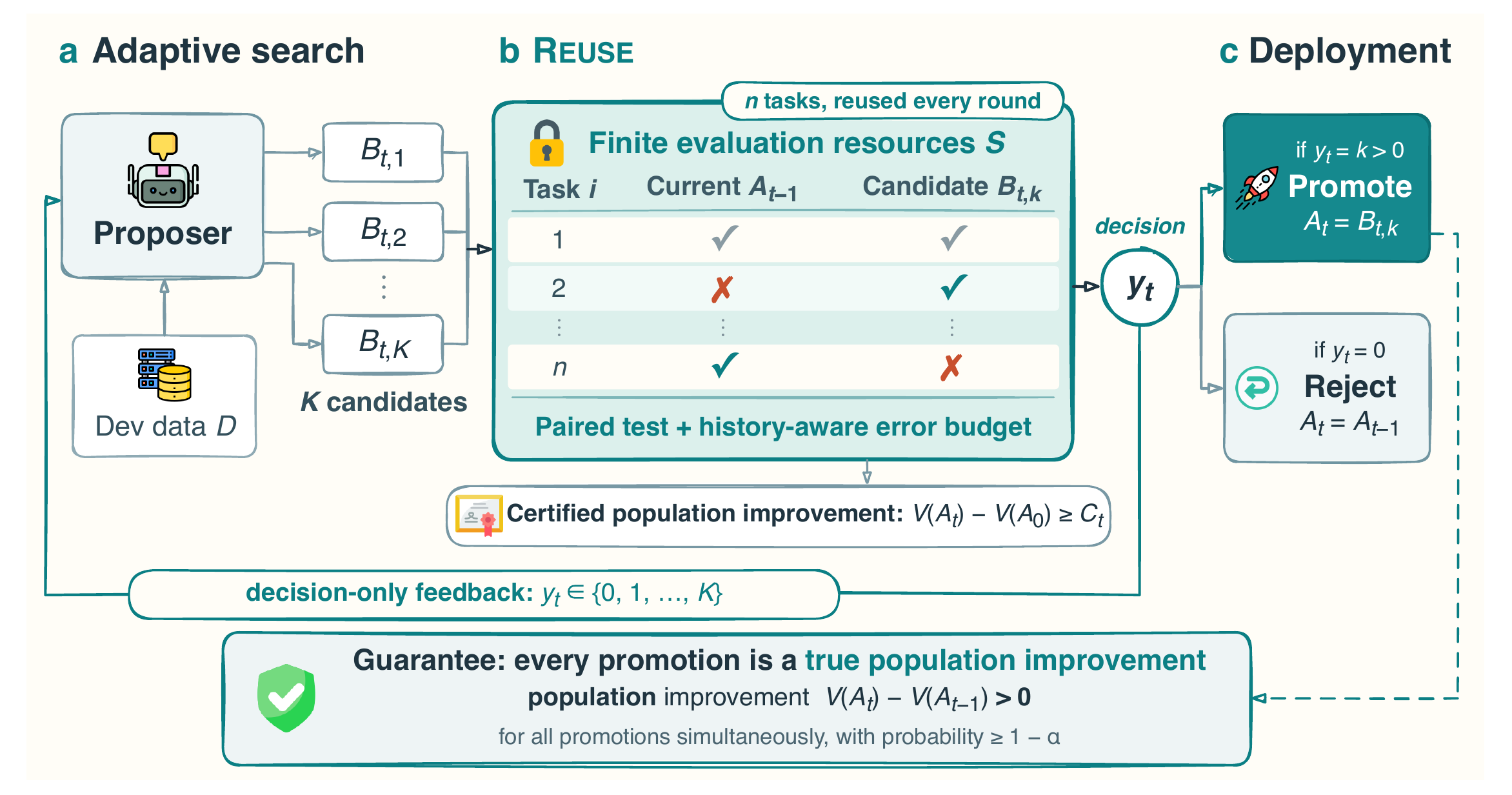}
    \caption{
Overview of \method{}.
(a) At round $t$, the proposer generates $K$ candidate modifications using development data and earlier promotion decisions.
(b) \method{} compares each candidate with the current system $A_{t-1}$ on a fixed evaluation set $S$ of $n$ tasks that is reused in every round.
(c) Only the promotion decision $y_t$ is returned to the proposer.
With probability at least $1-\alpha$, every promotion is a true improvement in population performance, $V(A_t)>V(A_{t-1})$, simultaneously over the entire run, and the cumulative improvement $V(A_t)-V(A_0)$ is certified to be at least $C_t$.
}
    \label{fig:interface}
\end{figure*}

In this paper, we propose \method{} (\textbf{R}isk-controlled \textbf{E}valuation \textbf{U}nder \textbf{S}equential \textbf{E}volution), a certified evaluation and promotion framework for adaptive self-improvement that controls all three sources jointly. 
\method{} allows a fixed evaluation set to be reused throughout an adaptive search while preserving valid evidence of population improvement. 
In each round, candidate modifications are evaluated against the current system, but the search process receives only the promotion decision rather than evaluation scores or task-level outcomes (\cref{fig:interface}). 
Later candidates may still adapt to earlier decisions, but they can depend on the evaluation set only through this decision history. 
For each possible history, the candidates are determined without reference to the evaluation set, so the corresponding comparison is a valid test. 
\method{} therefore turns the dependence on the evaluation set into a multiple comparison problem over all possible decision histories, and controls it together with the first two sources within a  global error budget. 
As a result, for a user-specified level $\alpha$, with probability at least $1-\alpha$, every promoted modification over the entire self-improvement process is a genuine population improvement on the underlying task distribution.

Our main contributions are:
\begin{itemize}
    \item We propose \method{}, a certified evaluation and promotion framework for adaptive self-improvement that jointly addresses the statistical challenges of within-round selection, repeated testing, and adaptive dependence under finite evaluation resources.
    For a user-specified error level $\alpha$, with probability at least $1-\alpha$, every promoted modification is a genuine population improvement on the underlying task distribution.

    \item We develop statistical theory for evaluation under benchmark reuse, including simultaneous error control over the entire process, valid lower bounds on cumulative improvement, the minimum detectable improvement, and the fundamental limits of adaptive evaluation reuse.

\item We evaluate \method{} in live self-improvement experiments against evaluation frameworks from current RSI systems and error-controlled baselines. \method{} commits substantially fewer false promotions than competing methods, reducing the proportion of promotions that are false from up to 20.7\% to 0\%, while achieving final true population performance within 0.04 percentage points of the best baselines, demonstrating that its rigorous statistical guarantees do not compromise overall progress.
\end{itemize}


\section{Methodology and Theoretical Guarantees}
\label{sec:method}

In this section, we formalize the adaptive self-improvement setting, present \method{}, and establish its statistical guarantees. 
More comprehensive theoretical results and complete proofs are given in \cref{app:proofs}.

\subsection{Setting and objective}
\label{sec:setting}

A system $A\in\mathcal A$ denotes the complete agent or AI system evaluated at a given round.
A modification may change its prompt, code, tools, training procedure, or model parameters.
Let $Q$ denote the underlying population distribution on which system performance is measured.
For an evaluation draw $W\sim Q$, the score of system $A$ is $s(A,W)\in[0,1]$.
Any randomness used in the evaluation is included in $W$.
The population performance of $A$ is
\[
V(A)=\E_{W\sim Q}[s(A,W)],
\qquad
\Delta(B,A)=V(B)-V(A),
\]
where $\Delta(B,A)$ is the population improvement of candidate $B$ over system $A$.

The self-improvement process starts from $A_0$. At round $t$, the current system is $A_{t-1}$ and the proposer
generates $K$ candidates
\(
B_{t,1},\ldots,B_{t,K}.
\)
The promotion decision is $y_t\in\{0,1,\ldots,K\}$. If $y_t=k>0$, candidate $B_{t,k}$ becomes the new system.
If $y_t=0$, all candidates are rejected and the incumbent is retained.
We write
\(
p_t=\sum_{s=1}^t\mathbf{1}\{y_s\neq0\}
\)
for the number of promotions through round $t$.

Promotion decisions are based on a fixed evaluation set
\[
S=(W_1,\ldots,W_n),
\qquad
W_i\overset{\mathrm{i.i.d.}}{\sim}Q,
\]
drawn before the process begins and repeatedly reused across rounds.
For any  pair of systems $(B,A)$, define the paired differences
\[
Z_i(B,A)=s(B,W_i)-s(A,W_i),
\]
which are i.i.d.\ with mean $\Delta(B,A)$.
The proposer receives information from $S$ only through previous promotion decisions.
It may otherwise adapt using its own randomness $R$, development data $D$, and those decisions.
We formalize this setting with two assumptions.

\begin{assumption}[Independent evaluation set]
\label{ass:indep}
The evaluation set $S$ is independent of the tuple $(R,D,A_0)$ consisting of the proposer's randomness,
the development data, and the initial system.
\end{assumption}

{
\begin{assumption}[Decision-only feedback]
\label{ass:iso}
Write $\xi=(R,D,A_0)$.
For every round $t\ge1$ and candidate index $k\le K$, there is a map $f_{t,k}$ such that
$B_{t,k}=f_{t,k}(\xi,y_{1:t-1})$.
Moreover, the map $(\xi,w)\mapsto s(A_0,w)$ and, for every $t\ge1$, $k\le K$, and fixed history
$h\in\{0,1,\ldots,K\}^{t-1}$, the map $(\xi,w)\mapsto s(f_{t,k}(\xi,h),w)$ are jointly measurable.
\end{assumption}

\Cref{ass:indep} requires that the evaluation set $S$ be drawn independently of everything that shapes the search from the start, which is exactly what a benchmark is meant to be, a fresh sample from the target task distribution that plays no role in development.
\Cref{ass:iso} requires that $S$ influence the search only through the promotion decisions, while the proposer is otherwise unrestricted and can freely use any language model or development pipeline.

Together, the two assumptions ensure that candidate generation depends on $S$ only through the decision history.
Once $\xi$ is fixed, every possible decision history determines the incumbent and the next candidates without reference to $S$, and by \cref{ass:indep} $S$ remains an i.i.d.\ sample from $Q$.
Every comparison determined in this way is therefore a valid test, and each comparison actually made during the search is one of them.
This turns the adaptive dependence created by benchmark reuse into a multiple comparison problem over all possible decision histories, which the allocation in \cref{sec:reuse} explicitly controls.
The measurability condition is technical and holds for essentially any practical proposer, for example whenever systems are prompts, programs, or configurations.
}

For a practical improvement threshold $\gamma\ge0$, a promotion is false if
\(
\Delta(A_t,A_{t-1})\le\gamma.
\)
Our goal is to control all promotions simultaneously:
\begin{equation}
\label{eq:global-objective}
\Prob\!\left(
\Delta(A_t,A_{t-1})>\gamma
\text{ for every }t\text{ with }y_t\neq0
\right)
\ge1-\alpha.
\end{equation}

\subsection{\method{}}
\label{sec:reuse}

To achieve the guarantee in \eqref{eq:global-objective}, \method{} tests each candidate against the current
system using paired evaluation and allocates a single error budget across all candidate comparisons that may arise
during the adaptive search. At each round, the proposer generates candidates using development data and previous
promotion decisions, and receives only the evaluation framework's promotion decision. \Cref{fig:interface} illustrates this process.
We first consider pass/fail benchmarks with $\gamma=0$, where $s(A,W)\in\{0,1\}$.
The same framework extends naturally to bounded scores $s(A,W)\in[0,1]$, where the detailed construction is given in \cref{app:paired-bound}.

For a fixed candidate $B$ and incumbent $A$, let $n_+$ be the number of evaluation tasks on which $B$ succeeds
and $A$ fails, let $n_-$ be the number on which $A$ succeeds and $B$ fails, and let
\(
M=n_++n_-
\)
be the number of tasks on which the two systems have different outcomes.
We test whether the candidate genuinely improves on the incumbent:
\begin{equation}
\label{eq:promotion-test}
H_0:\Delta(B,A)\le0
\qquad\text{versus}\qquad
H_1:\Delta(B,A)>0.
\end{equation}
Under $H_0$, conditional on a task where the two systems disagree, the probability that $B$ wins is at most
$1/2$. Conditional on $M$, this gives the exact one-sided paired sign-test $p$-value
\begin{equation}
\label{eq:signtest}
\pi(n_+,M)
=
\Prob\!\left(
\mathrm{Bin}(M,\tfrac12)\ge n_+
\right),
\qquad
\pi(0,0)=1.
\end{equation}
A small value of $\pi(n_+,M)$ provides evidence in favor of $\Delta(B,A)>0$, and
its finite-sample validity is established in \cref{lem:sign}.

A single hypothesis test controls the error for one fixed candidate--incumbent comparison. In the
self-improvement loop, however, many comparisons are made on the same evaluation set, and later candidates
depend on earlier promotion decisions. To control false promotions over the entire run, \method{} allocates
the global error budget $\alpha\in(0,1)$ across all comparisons that can arise under all possible decision histories.

Suppose $p$ promotions have occurred before round $t$. The $p$ promotion rounds can be chosen in
$\binom{t-1}{p}$ ways, and each promotion can select one of $K$ candidates. Hence there are
$\tbinom{t-1}{p}\,K^{p}$
possible decision histories with $p$ previous promotions. Under \cref{ass:iso}, once $(R,D,A_0)$ are fixed,
each such history determines the incumbent and the $K$ candidates at round $t$.
Let $(w_t)_{t\ge1}$ and $(v_p)_{p\ge0}$ be nonnegative weights satisfying
\(
\sum_{t\ge1}w_t\le1
\)
and
\(
\sum_{p\ge0}v_p\le1.
\)
These weights distribute the overall error budget across the round index and the number of previous promotions.
At round $t$ after $p$ promotions, each candidate comparison is assigned the error budget
\begin{equation}
\label{eq:alloc}
\delta_{t,p}
=
\frac{\alpha w_t v_p}
{\binom{t-1}{p}K^{p+1}}.
\end{equation}
This explicit allocation addresses all three sources of false promotion mentioned in \cref{sec:intro} simultaneously. The weights over rounds $t$ control repeated testing, the additional factor $K$ in the denominator accounts for within-round selection, and the $\binom{t-1}{p}K^p$ term handles adaptive dependence by covering all possible histories.
Throughout the paper we use
\(
w_t=1/[t(t+1)]
\)
and
\(
v_p=1/[(p+1)(p+2)].
\)

To also quantify cumulative improvement, we divide $\delta_{t,p}$ into two parts.
For a fixed $\rho\in(0,1)$, a fraction $\rho\delta_{t,p}$ is used for the sign test, and the remaining
$(1-\rho)\delta_{t,p}$ is used to construct a lower confidence bound $\ell$ on the candidate's population improvemetn
(\cref{lem:signlcb}). When a candidate is promoted, \method{} adds $\max\{\ell,0\}$ to a running
certificate $C_t$, initialized at $C_0=0$. The certificate is kept for auditing and is not returned to the proposer during the search.

For candidate $k$ at round $t$, define its empirical improvement over the incumbent as
\[
\bar Z_k
=
\frac1n\sum_{i=1}^n Z_i(B_{t,k},A_{t-1}).
\]
If multiple candidates pass the promotion test, \method{} promotes the one with the largest $\bar Z_k$.
The complete procedure is given in \cref{alg:gate}.

\begin{algorithm}[t]
\caption{\method{} at round $t$}
\label{alg:gate}
\small
\begin{algorithmic}[1]
\State $p\gets p_{t-1}$ and $\delta\gets\delta_{t,p}$
\For{$k=1,\ldots,K$}
    \State Compute $(n_{+,k},n_{-,k})$ and $\bar Z_k$ for $B_{t,k}$ against $A_{t-1}$
    \State $\mathrm{pass}_k
    \gets
    \left[
    \pi(n_{+,k},n_{+,k}+n_{-,k})
    \le
    \rho\delta
    \right]$
\EndFor
\If{at least one candidate passes}
    \State $k^\star\gets\arg\max_{k:\mathrm{pass}_k}\bar Z_k$
    \State Compute a lower confidence bound $\ell$ for
    $\Delta(B_{t,k^\star},A_{t-1})$ at error level $(1-\rho)\delta$
    \State $y_t\gets k^\star$
    \State $C_t\gets C_{t-1}+\max\{\ell,0\}$
\Else
    \State $y_t\gets0$, $C_t\gets C_{t-1}$
\EndIf
\end{algorithmic}
\end{algorithm}

Conditional on $(R,D,A_0)$, each fixed-history comparison has probability at most $\delta_{t,p}$ of either
an erroneous promotion decision or an invalid lower confidence bound.
A union bound over all rounds, all possible histories, and all candidate comparisons gives
\[
\sum_{t\ge1}
\sum_{p=0}^{t-1}
\binom{t-1}{p}K^p
\cdot K\delta_{t,p}
=
\alpha
\sum_{t\ge1}w_t
\sum_{p=0}^{t-1}v_p
\le\alpha.
\]
This gives the following main result.

{
\begin{theorem}[Simultaneous validity and cumulative progress]
\label{thm:main}
Consider pass/fail evaluation, with $s(A,W)\in\{0,1\}$ and $\gamma=0$.
Suppose \cref{ass:indep,ass:iso} hold, and \method{} is run as in
\cref{alg:gate} with the allocation \eqref{eq:alloc}.
Then, for every adaptive proposer satisfying these assumptions,
with probability at least $1-\alpha$,
$\Delta(A_t,A_{t-1})>0$ for every $t$ with $y_t\neq0$.
On the same event, the running certificate satisfies
$V(A_t)-V(A_0)\ge C_t$ for every $t\ge1$.
\end{theorem}
}

The cumulative bound follows by summing the valid stepwise contributions $\max\{\ell,0\}$ over promotions,
since the corresponding population improvement telescope to $V(A_t)-V(A_0)$.

Development data may also be used to select a single candidate before $S$ is queried.
If only that candidate is tested on $S$, \cref{thm:main} applies with $K=1$ in
\eqref{eq:alloc}, since there is no within-round candidate selection on the evaluation set.

We now establish how large a true improvement must be for reliable detection.
We let $\cL_{t,p} = \ln({2}/{\delta_{t,p}})$ and let
\(
d=\Prob\!\left(Z_i(B,A)\neq0\right)
\)
denote the probability that the candidate and incumbent have different outcomes on a random task.
The following result gives the corresponding minimum detectable population improvement.

\begin{proposition}[Minimum detectable improvement]
\label{prop:power}
Consider pass/fail evaluation with $\gamma=0$ and $\rho=1/2$ at a fixed decision history node.
Suppose the candidate and incumbent pairs are fixed independently of $S$ and each has disagreement
probability $d\in(0,1]$. Define
\[
\Delta_{\min}(n,t,p,K,d)
=
\sqrt{\frac{2d\,\cL_{t,p}}{n}}.
\]
Under the regularity conditions in \cref{app:power}, for any fixed $\eta\in(0,1)$,
a candidate with $\Delta(B,A)\ge(1+\eta)\Delta_{\min}$ {passes the promotion test} with probability tending to one,
whereas if all $K$ candidates satisfy
$\Delta(B_{t,k},A_{t-1})\le(1-\eta)\Delta_{\min}$, the probability of any promotion tends to zero.
\end{proposition}

The quantity $\Delta_{\min}$ represents the smallest population improvement that we can reliably detect.
The formula shows that a larger evaluation set $n$ in the denominator decreases this detection threshold.
At the same time the term $\cL_{t,p}$ in the numerator depends on the per-round error level $\delta_{t,p}$.
To maintain our global guarantee we must allocate smaller error budgets $\delta_{t,p}$ as the round count $t$ and promotion count $p$ increase.
This decrease in $\delta_{t,p}$ causes $\cL_{t,p}$ to grow which makes larger true improvements necessary for reliable detection at later stages.
This behavior directly reflects our statistical intuition.
By demanding stronger evidence as the search proceeds we prevent random noise from accumulating into false promotions over a long trajectory.

This dependence on the adaptive history is not solely an artifact of our particular allocation.
In \cref{app:lb} we show that for a class of threshold rules a comparable
first order growth in the promotion threshold is necessary to prevent an adaptive proposer from
eventually inducing a false promotion.



\section{Experiments}
\label{sec:experiments}
In this section, we evaluate \method{} against existing methods in live self-improvement loops, where an LLM-based proposer generates modifications and a finite evaluation set is repeatedly reused. 
Our experiments address two key questions: First, how often do current RSI promotion schemes accept false improvements? Second, can \method{} prevent these false promotions without sacrificing true population improvement?

\subsection{Experimental setup}
\label{sec:exp-setup}
We consider the agent tasked with machine learning classification. 
Specifically, we use the Covertype dataset \citep{blackard1999covertype}\footnote{\url{https://archive.ics.uci.edu/dataset/31/covertype}} for the binary classification of cover type 2 against the remaining six classes.
A fixed random split yields a training set $D_{\mathrm{tr}}$ of 50,000 examples for fitting candidates, a development set $D_{\mathrm{dev}}$ of 20,000 examples for proposer feedback, and an evaluation pool $\mathcal E$ of 60,000 examples, from which each run draws an evaluation set $S$ of size $n\in\{2{,}000,10{,}000\}$ that is reused throughout the search.
Together, $D_{\mathrm{tr}}$ and $D_{\mathrm{dev}}$ form  $D$ of \cref{sec:setting}.
Since the underlying task distribution is unknown, we measure population  improvement on a disjoint held-out set $H$ of 50,000 examples.
Each run lasts $T=200$ rounds. 
At every round, Qwen2.5-7B-Instruct \citep{qwen25} proposes up to $K=8$ modifications of a scikit-learn pipeline \citep{pedregosa2011scikit} based on the prompt templates detailed in \cref{app:experiment-prompts}.
Each modification changes at most two components among the model family, feature subset, pre-processing, and hyperparameters.
All methods start from the same gradient-boosting configuration, share the proposer, candidate space, and
training data, and see the same evaluation set within a seed.
We run {30 seeds} at each evaluation-set size, and \cref{app:experiments} gives the full protocol.

We compare \method{} with several evaluation frameworks that reflect how current RSI systems use their benchmarks. Empirical best-of-$K$ edits the incumbent and promotes the candidate with the largest positive gain on $S$ \citep{anthropic2026aar}. 
SICA-style follows the archive-based search of SICA \citep{sica2025}, DGM-style follows the parent selection of DGM \citep{dgm2025}, and Elite and Niche-elite use population-based search inspired by AlphaEvolve \citep{alphaevolve2025}. 
We also evaluate baselines that add statistical control to the same loop. 
PACE-style applies a one-sided McNemar test \citep{dietterich1998approximate} to each candidate at level $\alpha=0.05$, a fixed-sample analogue of the per-decision error control in PACE \citep{pace2026}.
Bonferroni-style applies the same test at level $0.05/(TK)$, splitting the error budget evenly over all potential comparisons of a run, as in the union-bound allocation of SGM \citep{sgm2025}.

To measure true improvement, we estimate the population improvement of each promoted modification by its paired improvement on the held-out set $H$, which no method ever accesses, and call a promotion false when this improvement is at most zero.
For each method we report the final population improvement of the deployed system over the starting system $A_0$, the mean number of promotions per run, and the total number of false promotions across runs.
We also report the optimism gap of the final system, defined as its improvement over $A_0$ measured on $S$ minus its population improvement, which measures how much the reused evaluation set overstates the final improvement.

\subsection{Results}
\label{sec:exp-results}

\begin{table}[t]
\caption{
Live self-improvement on Covertype over 30 seeds at each evaluation-set size, with all methods run on the same seeds.
Pop.\ Impr.\ is the final population improvement of the deployed system over starting system $A_0$ in percentage points (pp), estimated on the held-out set $H$ and reported as mean $\pm$ standard error.
Prom.\ is the mean number of promotions per run, and False is the number of promotions with non-positive population improvement, summed over the 30 runs.
Opt.\ Gap is the final improvement measured on $S$ minus the population improvement measured on $H$, which measures how much the reused evaluation set overstates the final improvement.
Select., Repeat., and Adapt.\ indicate whether a rule controls errors from within-round selection, repeated testing, and adaptive dependence, the three sources of false promotion discussed in \cref{sec:intro}.
}
\label{tab:live-results}
\centering
\small
\setlength{\tabcolsep}{3.6pt}
\renewcommand{\arraystretch}{1.15}
\resizebox{\textwidth}{!}{%
\begin{tabular}{@{}lccc cccc c cccc@{}}
\toprule
& & &
& \multicolumn{4}{c}{$n=2{,}000$}
&
& \multicolumn{4}{c}{$n=10{,}000$} \\
\cmidrule(lr){5-8}\cmidrule(lr){10-13}
Method & Select. & Repeat. & Adapt.
& Pop.\ Impr.\ $\uparrow$ & Prom. & False $\downarrow$ & Opt.\ Gap $\downarrow$
&
& Pop.\ Impr.\ $\uparrow$ & Prom. & False $\downarrow$ & Opt.\ Gap $\downarrow$ \\
\midrule
Empirical best-of-$K$
  & \xmark & \xmark & \xmark
  & \pmse{7.04}{0.08} & 13.0 & 75 & \pmse{0.43}{0.11}
  &
  & \pmse{7.19}{0.06} & 14.2 & 71 & \pmse{0.15}{0.06} \\
SICA-style
  & \xmark & \xmark & \xmark
  & \pmse{6.36}{0.18} & 12.3 & 72 & \pmse{0.34}{0.10}
  &
  & \pmse{6.56}{0.16} & 12.5 & 62 & \pmse{0.14}{0.05} \\
DGM-style
  & \xmark & \xmark & \xmark
  & \pmse{3.89}{0.17} & 6.4 & 34 & \pmse{0.06}{0.11}
  &
  & \pmse{4.03}{0.18} & 7.2 & 18 & \pmse{0.01}{0.05} \\
Elite
  & \xmark & \xmark & \xmark
  & \pmse{6.72}{0.14} & 14.3 & 89 & \pmse{0.47}{0.12}
  &
  & \pmse{6.85}{0.12} & 13.7 & 59 & \pmse{0.16}{0.05} \\
Niche-elite
  & \xmark & \xmark & \xmark
  & \pmse{6.46}{0.15} & 12.3 & 71 & \pmse{0.40}{0.13}
  &
  & \pmse{6.37}{0.15} & 12.4 & 44 & \pmse{0.13}{0.05} \\
PACE-style
  & \xmark & \xmark & \xmark
  & \pmse{6.45}{0.19} & 6.4 & 4 & \pmse{0.31}{0.10}
  &
  & \pmse{7.23}{0.02} & 8.1 & 3 & \pmse{0.10}{0.05} \\
Bonferroni-style
  & \cmark & \cmark & \xmark
  & \pmse{2.35}{0.37} & 1.3 & 0 & \pmse{0.21}{0.08}
  &
  & \pmse{6.90}{0.12} & 6.5 & 0 & \pmse{0.10}{0.05} \\
\addlinespace[3pt]
\rowcolor{oursbg}
\method{} (ours)
  & \cmark & \cmark & \cmark
  & \pmse{7.02}{0.26} & 3.2 & \good{0} & \pmse{0.05}{0.12}
  &
  & \pmse{7.19}{0.13} & 5.9 & \good{0} & \pmse{0.06}{0.07} \\
\bottomrule
\end{tabular}%
}
\end{table}

\begin{figure}[t]
\centering
\includegraphics[width=0.95\textwidth]{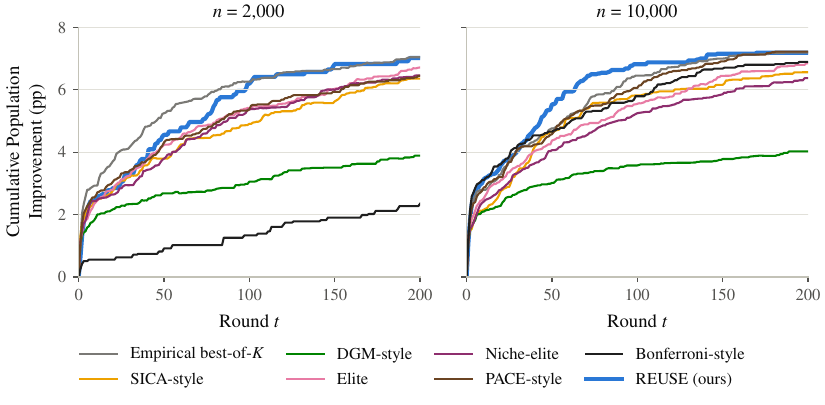}
\caption{
Population improvement of the deployed system over starting system $A_0$ across rounds, estimated on $H$ and averaged over 30 seeds, in percentage points (pp).
The left panel uses $n=2{,}000$ and the right panel uses $n=10{,}000$.
}
\label{fig:trajectory}
\end{figure}

\begin{figure}[t]
\centering
\includegraphics[width=0.95\textwidth]{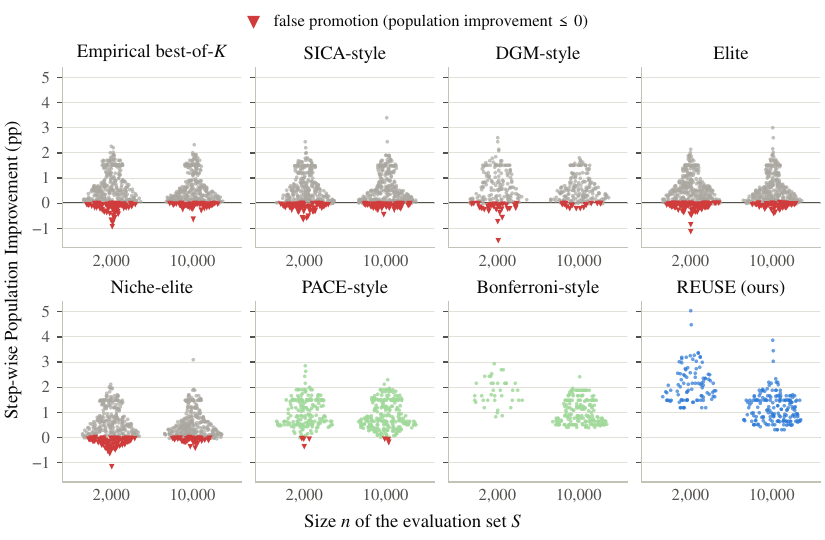}
\caption{
Population improvement of every promoted modification over the system it replaces, estimated on $H$ and pooled over 30 seeds, in percentage points (pp).
Each point is one promotion, and red triangles mark false promotions, whose population improvement is at most zero.
Within each panel, the left column uses $n=2{,}000$ and the right column uses $n=10{,}000$.
}
\label{fig:promotion-gains}
\end{figure}

\Cref{tab:live-results} reports the final population improvement over $A_0$ for each method alongside the number of promotions, the count of false promotions, and the optimism gap (Opt.\ Gap) of the final system.
The results highlight how existing evaluation methods fail and how \method{} guarantees reliability without sacrificing true performance.
For instance, at $n=2{,}000$, Empirical best-of-$K$ makes 75 false promotions, that is, nearly 20 percent of its accepted updates yield non-positive population improvement.
Even when the evaluation set expands to $n=10{,}000$, this framework still commits 71 false promotions.
Although Empirical best-of-$K$ achieves high final population improvement, its gains only match those of \method{}.
Other benchmark-driven frameworks like SICA-style also commit numerous false promotions while yielding substantially lower population improvements.

The Opt.\ Gap column further clarifies these behaviors.
Methods like Empirical best-of-$K$ and PACE-style show large gaps because they adaptively overfit to the evaluation set.
DGM-style exhibits a small gap but achieves very low population improvement, which reflects an inability to optimize rather than successful generalization.
Statistical baselines also struggle to balance safety and progress.
Bonferroni-style avoids false promotions but proves heavily conservative.
It makes only 1.3 promotions per run at $n=2{,}000$ and restricts its final population improvement to a mere 2.35 pp, well below the 7.02 pp achieved by \method{}.
PACE-style improves performance but lacks formal guarantees and still commits false promotions at both evaluation set sizes.
\method{} effectively solves this trade-off.
It makes no false promotion while matching the strongest methods, achieving a population improvement of 7.02 pp at $n=2{,}000$ and 7.19 pp at $n=10{,}000$.
Furthermore, its optimism gap stays within one standard error of zero at both sizes alongside high final true performance, proving that it eliminates adaptive overfitting without falling into the underfitting trap seen in DGM-style.

The trajectories across rounds in \cref{fig:trajectory} show that \method{} consistently stays near the top of the performance curves and establishes dominance when the sample size increases to $n=10{,}000$. 
\Cref{fig:promotion-gains} reveals the step-wise dynamics of individual promotions. Most competing methods display a dense cluster of points below zero, visibly marked by red triangles that represent false promotions. Bonferroni-style has no red triangles but produces very few points, confirming its overly conservative nature. In contrast, the entire cluster of points for \method{} sits noticeably higher, and every promotion is strictly positive.


\section{Related Work}
\label{sec:related}

\paragraph{Recursive self-improvement and statistical control.}
The idea of a system that improves itself with guarantees dates back to the G{\"o}del machine \citep{schmidhuber2007godel}. Language models have turned this into practical search loops that rewrite prompts, agent scaffolds, algorithms, and entire pipelines \citep{zelikman2024stop, fernando2024promptbreeder, sica2025, dgm2025, alphaevolve2025, anthropic2026aar, rsisurvey2026}.
However, evaluating these systems by repeatedly querying fixed benchmarks often leads to false promotions and degraded deployment performance \citep{wang2026rethinking, pace2026, guo2026seal, wang2025statlimits}.
To address this, recent works gate self-modification with statistical tests, utilizing global error budgets \citep{sgm2025, sea2026}, sequential testing \citep{pace2026}, or sealed audits \citep{guo2026seal, mittal2026signed}.
Yet, these approaches typically treat candidates as independent of the test data or restrict feedback without accounting for the information revealed by the promotion decisions themselves.
\method{} overcomes this by combining decision-only feedback with an explicit error allocation over all possible decision histories, thereby controlling adaptive dependence on a  reused evaluation set over the entire run.

\paragraph{Adaptive data analysis and holdout reuse.}
The risk of reusing a holdout set adaptively is a central topic in adaptive data analysis. 
Mechanisms based on differential privacy and description length allow a holdout set to answer adaptively chosen queries \citep{dwork2015reusable, dwork2015generalization}, and the Ladder mechanism maintains a leaderboard by releasing scores only upon significant improvement \citep{blum2015ladder}.
Information-theoretic arguments bound the bias of adaptive analysis \citep{russo2016controlling, xu2017information, bassily2016algorithmic}, though empirical benchmark reuse often shows less overfitting than worst-case theory predicts \citep{recht2019imagenet, roelofs2019meta, bertran2026fits}.
\method{}  explicitly counts decision histories, turning the information bound into exact finite-sample tests for each promotion.

\paragraph{Multiple testing and sequential inference.}
Controlling the family-wise error over many hypotheses encompasses classical \citep{holm1979simple}, weighted \citep{bretz2009graphical}, and sequential anytime-valid procedures \citep{lan1983discrete, javanmard2018online, howard2021time, ramdas2023game, waudbysmith2024betting}.
Similarly, selective inference adjusts tests for data-driven selection biases \citep{berk2013valid, fithian2014optimal, efron2011tweedie, jamieson2018bandit}.
However, these frameworks generally assume that hypotheses are fixed in advance, arrive with fresh data, or are selected by characterizable distributions.
In a self-improvement loop, hypotheses are instead generated by an arbitrary proposer adapting to earlier outcomes on the same data.
\method{} bridges this gap by allocating an error budget over all possible decision histories, ensuring each test is valid in finite samples \citep{clopper1934,maurer2009empirical}.

\section{Discussion}
\label{sec:discussion}

Self-improving systems increasingly decide what to keep by querying the same evaluation set over long trajectories.
We showed that this reuse creates a source of error that multiple comparison corrections do not address, because later candidates depend on earlier outcomes on the evaluation set.
\method{} controls this by returning only promotion decisions and allocating an  error budget over all decision histories, which yields a simultaneous guarantee that every promoted modification is a true population improvement.
In live self-improvement experiments, \method{} made much fewer false promotion than other frameworks without sacrificing true population performance.
\method{} achieves this reliability while placing no restrictions on the underlying proposer or agent architecture. While our experiments demonstrate the framework on a machine learning classification task, extending it to larger agentic settings like coding and model post-training is a natural next step. The decision-only feedback currently provides strict statistical guarantees, but future work could accommodate richer signals including rounded scores, privatized summaries, or periodically refreshed evaluation sets by explicitly accounting for their information transcripts. Furthermore, expanding the framework to simultaneously control for secondary metrics like safety and cost offers a promising path toward generating the auditable and provable records of capability gains increasingly required by frontier AI safety frameworks \citep{openai2025preparedness}.

\bibliographystyle{iclr2027_conference}
\bibliography{references}

\appendix


\numberwithin{theorem}{section}
\numberwithin{lemma}{section}
\numberwithin{proposition}{section}
\numberwithin{corollary}{section}
\numberwithin{assumption}{section}
\numberwithin{definition}{section}
\numberwithin{remark}{section}
\renewcommand{\thetheorem}{\thesection\arabic{theorem}}
\renewcommand{\thelemma}{\thesection\arabic{lemma}}
\renewcommand{\theproposition}{\thesection\arabic{proposition}}
\renewcommand{\thecorollary}{\thesection\arabic{corollary}}
\renewcommand{\theassumption}{\thesection\arabic{assumption}}
\renewcommand{\thedefinition}{\thesection\arabic{definition}}
\renewcommand{\theremark}{\thesection\arabic{remark}}

\newpage
\section{Proofs and Additional Theory}
\label{app:proofs}


\subsection{A paired lower confidence bound for bounded scores}
\label{app:paired-bound}

We first give a lower confidence bound for the mean of bounded paired differences. This result is used for
non-binary scores and as one component of the certificate for pass/fail evaluation.

Let $z_1,\ldots,z_n\in[-1,1]$. Write
\[
\bar z=\frac1n\sum_{i=1}^n z_i,
\qquad
\overline{z^2}=\frac1n\sum_{i=1}^n z_i^2.
\]
Here $\kl(a\,\|\,b)=a\ln(a/b)+(1-a)\ln((1-a)/(1-b))$ is the Bernoulli relative entropy,
with the usual continuous extensions at the endpoints. All logarithms are natural.
For $L>0$, define
\begin{equation}
\label{eq:margin}
\hat s^+(z;L)
=
\max\left\{
q\in[\overline{z^2},1]:
n\,\kl(\overline{z^2}\,\|\,q)\le L
\right\},
\end{equation}
and
\begin{equation}
\label{eq:epsilon}
\varepsilon(v,L)
=
\frac{2L}{3n}
+
\sqrt{
\left(\frac{2L}{3n}\right)^2
+
\frac{2vL}{n}
}.
\end{equation}
We use the lower bound
\begin{equation}
\label{eq:lcb}
\mathrm{lcb}(z;L)
=
\bar z-\varepsilon(\hat s^+(z;L),L).
\end{equation}

The quantity $\hat s^+(z;L)$ is an upper confidence bound for $\E Z^2$, which in turn upper bounds
$\mathrm{Var}(Z)$. The second term in \eqref{eq:lcb} is a Bernstein deviation evaluated at this data-dependent
variance proxy.

\begin{lemma}[Paired lower confidence bound]
\label{lem:bkl}
Let $Z_1,\ldots,Z_n$ be i.i.d.\ random variables in $[-1,1]$ with mean $\mu$. For
$\delta\in(0,1)$ and $L=\ln(2/\delta)$,
\[
\Prob\!\left(
\mu<\mathrm{lcb}(Z_{1:n};L)
\right)
\le\delta.
\]
\end{lemma}

\begin{proof}
Let
\[
m=\E Z_1^2,
\qquad
\sigma^2=\mathrm{Var}(Z_1)\le m.
\]

{Since $Z_i^2\in[0,1]$ are i.i.d.\ with mean $m$, applying the
Chernoff--Hoeffding bound to $1-Z_i^2$
\citep{hoeffding1963} gives
\[
\Prob\!\left(\overline{Z^2}\le a\right)
\le \exp\{-n\kl(a\,\|\,m)\},
\qquad 0\le a<m.
\]
For $a=0$ and $m<1$, the same bound follows from
$\Prob(Z_i^2=0)\le1-m$ and independence. If $m=1$, then
$Z_i^2=1$ almost surely, so the probability is zero for every $a<1$,
consistent with the convention $\exp(-\infty)=0$.}
By the definition of $\hat s^+$,
\[
\Prob(\hat s^+<m)\le e^{-L}=\delta/2.
\]

{Next, the centered variables satisfy $Z_i-\mu\le2$. Bernstein's
inequality \citep{hoeffding1963}, with range parameter $b=2$,
gives, for every $\epsilon>0$,
\[
\Prob(\bar Z-\mu\ge\epsilon)
\le
\exp\!\left\{-\frac{n\epsilon^2}
{2\sigma^2+\frac43\epsilon}\right\}.
\]
If $\sigma^2=0$, the left-hand side is zero and the inequality remains valid.}
The positive solution of
\[
n\epsilon^2
=
L\left(2\sigma^2+\frac43\epsilon\right)
\]
is $\varepsilon(\sigma^2,L)$. Hence
\[
\Prob\!\left(
\bar Z-\mu
\ge
\varepsilon(\sigma^2,L)
\right)
\le
e^{-L}
=
\delta/2.
\]

The function $\varepsilon(v,L)$ is increasing in $v$. On the event $\{\hat s^+\ge m\}$,
\[
\varepsilon(\hat s^+,L)
\ge
\varepsilon(m,L)
\ge
\varepsilon(\sigma^2,L).
\]
Therefore
\[
\left\{
\mu<
\mathrm{lcb}(Z_{1:n};L)
\right\}
\subseteq
\{\hat s^+<m\}
\cup
\left\{
\bar Z-\mu>
\varepsilon(\sigma^2,L)
\right\}.
\]
The two probabilities sum to at most $\delta$.
\end{proof}


\subsection{Validity of the exact paired sign test}

For pass/fail evaluation,
\[
Z_i\in\{-1,0,+1\}.
\]
Let
\[
q_+=\Prob(Z_i=+1),
\qquad
q_-=\Prob(Z_i=-1),
\qquad
d=q_++q_-.
\]
The population gain is
\[
\mu=\E Z_i=q_+-q_-.
\]

Recall that $n_+$ and $n_-$ denote the number of candidate wins and losses and
\[
M=n_++n_-.
\]
The paired sign-test $p$-value is
\[
\pi(n_+,M)
=
\Prob\!\left(
\mathrm{Bin}(M,\tfrac12)\ge n_+
\right),
\qquad
\pi(0,0)=1.
\]

\begin{lemma}[Validity of the exact paired test]
\label{lem:sign}
Let $Z_1,\ldots,Z_n$ be i.i.d.\ in $\{-1,0,+1\}$ with mean $\mu\le0$. For any
$\delta'\in(0,1)$,
\[
\Prob\!\left(
\pi(n_+,M)\le\delta'
\right)
\le\delta'.
\]
The same inequality holds conditionally on every value of $M$.
\end{lemma}

{
\begin{proof}
If $d=0$, then $M=0$ almost surely and $\pi(0,0)=1$. The test never rejects.

Assume $d>0$ and define
\[
q=\frac{q_+}{d}=\Prob(Z_i=+1\mid Z_i\neq0).
\]
By independence, $(n_+,n_-,n-M)$ has the multinomial distribution
with cell probabilities $(q_+,q_-,1-d)$. For $0\le k\le m\le n$,
its probability mass factors as
\[
\begin{aligned}
\Prob(M=m,n_+=k)
&=\binom nm\binom mk q_+^kq_-^{m-k}(1-d)^{n-m}\\
&=\Prob(M=m)\binom mk q^k(1-q)^{m-k}.
\end{aligned}
\]
Thus, for every $m$ with $\Prob(M=m)>0$,
$n_+\mid M=m\sim\mathrm{Bin}(m,q)$.
The null hypothesis $\mu\le0$ is equivalent to $q\le1/2$.

Fix such an $m$. Under the null, $\mathrm{Bin}(m,q)$ is stochastically
dominated by $\mathrm{Bin}(m,1/2)$, so the rejection probability of
the upper-tail sign test is maximized over $q\le1/2$ at $q=1/2$.
Indeed, if there is a rejection threshold, let $k_0$ be the smallest
integer in $\{0,\ldots,m\}$ with $\pi(k_0,m)\le\delta'$.
Monotonicity of $\pi(k,m)$ in $k$ then gives
\[
\begin{aligned}
\Prob\!\left(\pi(n_+,M)\le\delta'\mid M=m\right)
&=\Prob\!\left(n_+\ge k_0\mid M=m\right)\\
&\le\Prob\!\left(\mathrm{Bin}(m,\tfrac12)\ge k_0\right)
=\pi(k_0,m)\le\delta'.
\end{aligned}
\]
If there is no rejection threshold, the rejection probability is zero.
For $m=0$ it is also zero. Averaging over $M$ proves the unconditional claim.
\end{proof}
}


\subsection{A lower confidence bound for pass/fail evaluation}

For $d>0$, define
\[
q=\frac{q_+}{d}.
\]
Then
\begin{equation}
\label{eq:mudq}
\mu
=
q_+-q_-
=
(2q-1)d,
\qquad
\E Z_i^2=d.
\end{equation}

{For $k$ successes in $m$ Bernoulli trials, let
$\underline\theta(k;m,a)$ denote the one-sided Clopper--Pearson lower
confidence limit at error level $a\in(0,1)$ \citep{clopper1934}.
For $1\le k\le m$, it is the $a$-quantile of
$\mathrm{Beta}(k,m-k+1)$, and we set $\underline\theta(0;m,a)=0$.
}

Given an error level $\delta_c\in(0,1)$, define
\[
q_{\mathrm{lo}}
=
\underline\theta(n_+;M,\delta_c/4),
\qquad
d_{\mathrm{lo}}
=
\underline\theta(M;n,\delta_c/4),
\]
and
\begin{equation}
\label{eq:sign-lcb}
\ell^{\mathrm{sgn}}
=
\begin{cases}
(2q_{\mathrm{lo}}-1)d_{\mathrm{lo}},
&
q_{\mathrm{lo}}>1/2,
\\[2mm]
-\infty,
&
q_{\mathrm{lo}}\le1/2.
\end{cases}
\end{equation}

We combine this exact bound with the bounded-score lower confidence bound from \cref{lem:bkl}:
\begin{equation}
\label{eq:ellstar}
\ell^\star(Z_{1:n};\delta_c)
=
\max\left\{
\mathrm{lcb}
\left(
Z_{1:n};
\ln\frac{4}{\delta_c}
\right),
\ell^{\mathrm{sgn}}
\right\}.
\end{equation}

\begin{lemma}[Lower confidence bound for pass/fail evaluation]
\label{lem:signlcb}
For i.i.d.\ $Z_1,\ldots,Z_n\in\{-1,0,+1\}$ with mean $\mu$,
\[
\Prob(\mu<\ell^{\mathrm{sgn}})
\le
\frac{\delta_c}{2},
\]
and
\[
\Prob\!\left(
\mu<\ell^\star(Z_{1:n};\delta_c)
\right)
\le
\delta_c.
\]
\end{lemma}

\begin{proof}
If $d=0$, then $M=n_+=0$ almost surely, so $q_{\mathrm{lo}}=0$ and $\ell^{\mathrm{sgn}}=-\infty$.
The first claim is immediate, and the second follows from \cref{lem:bkl}. Assume henceforth that $d>0$.

{By the conditional-count calculation in the proof of \cref{lem:sign},
the conditional law of $n_+$ given $M=m$ is
$\mathrm{Bin}(m,q)$ whenever $\Prob(M=m)>0$. Applying the one-sided
Clopper--Pearson coverage property conditionally and then averaging over $M$ gives
\[
\Prob(q<q_{\mathrm{lo}})\le\delta_c/4.
\]
Separately, the indicators $\mathbf{1}\{Z_i\neq0\}$ are i.i.d.\
Bernoulli variables with success probability $d$, so $M\sim\mathrm{Bin}(n,d)$.
Applying the same coverage property directly to this count gives
\[
\Prob(d<d_{\mathrm{lo}})\le\delta_c/4.
\]}

Therefore, with probability at least $1-\delta_c/2$,
\[
q\ge q_{\mathrm{lo}},
\qquad
d\ge d_{\mathrm{lo}}.
\]

On this event, when $q_{\mathrm{lo}}>1/2$,
\[
\mu
=
(2q-1)d
\ge
(2q_{\mathrm{lo}}-1)d_{\mathrm{lo}}
=
\ell^{\mathrm{sgn}}.
\]
When $q_{\mathrm{lo}}\le1/2$, the bound equals $-\infty$ and the inequality holds automatically. Hence
\[
\Prob(\mu<\ell^{\mathrm{sgn}})
\le
\delta_c/2.
\]

By \cref{lem:bkl},
\[
\Prob\left(
\mu<
\mathrm{lcb}
\left(
Z_{1:n};
\ln\frac4{\delta_c}
\right)
\right)
\le
\delta_c/2.
\]
A union bound over the two lower confidence bounds gives
\[
\Prob\!\left(
\mu<
\ell^\star(Z_{1:n};\delta_c)
\right)
\le
\delta_c.
\]
\end{proof}

\begin{remark}[Running cumulative bound]
\label{rem:silent}
The sign test may establish that $\mu>0$ even when the numerical lower confidence bound
$\ell^\star$ is non-positive. For the pass/fail test with $\gamma=0$, \method{} therefore adds
\[
\max\{\ell^\star,0\}
\]
to the running bound $C_t$. On the event of \cref{thm:main}, both
\[
\Delta(A_t,A_{t-1})>0
\]
and
\[
\Delta(A_t,A_{t-1})\ge\ell^\star
\]
hold. Hence
\[
\Delta(A_t,A_{t-1})
\ge
\max\{\ell^\star,0\}.
\]
This keeps $C_t$ nondecreasing while preserving its lower-bound interpretation.
For the bounded-score extension, promotion requires $\ell>\gamma\ge0$, and the contribution is $\ell$ itself.
\end{remark}


\subsection{Proof of the main guarantee}

We prove \cref{thm:main} for general allocation weights $(w_t)$ and $(v_p)$ satisfying
\[
\sum_{t\ge1}w_t\le1,
\qquad
\sum_{p\ge0}v_p\le1.
\]
For pass/fail evaluation with $\gamma=0$, let $\rho\in(0,1)$ denote the fraction of $\delta_{t,p}$ used for the
promotion test; the remaining fraction is used for the lower confidence bound. The proof applies to every such $\rho$.
If $\delta_{t,p}=0$, all candidates are rejected and the running bound is unchanged. Such comparisons can be omitted
from the argument below.

\begin{proof}[Proof of \cref{thm:main}]
Condition on the search procedure's randomness $R$, development data $D$, and starting system $A_0$.
By \cref{ass:indep}, $S$ remains i.i.d.\ from $Q$. All probabilities in this proof are conditional on $(R,D,A_0)$.

For round $t$, let
\[
y=y_{1:t-1}
\]
denote a possible decision history. By \cref{ass:iso}, every such history determines an incumbent
\[
a(y)
\]
and candidates
\[
b_1(y),\ldots,b_K(y)
\]
that do not otherwise depend on $S$. Here $y$ is a fixed hypothetical history: we do not condition on the event
that the evaluation set produces that history.

Let $\mathcal Y_{t-1,p}$ be the set of decision histories of length $t-1$ containing exactly $p$ promotions. Then
\begin{equation}
\label{eq:history-count}
|\mathcal Y_{t-1,p}|
=
\binom{t-1}{p}K^p.
\end{equation}
The factor $\binom{t-1}{p}$ selects the promotion rounds and the factor $K^p$ selects the promoted candidate at
each such round.

Fix $t$, $p$, $y\in\mathcal Y_{t-1,p}$, and candidate $k$. Write
\[
\mu_{y,k}
=
\Delta(b_k(y),a(y)).
\]

In what follows, $Z_{1:n}$, $n_+$, and $M$ refer to this fixed pair $(b_k(y),a(y))$.
For pass/fail evaluation with $\gamma=0$, define the decision failure event
\[
E^{\mathrm{dec}}_{t,y,k}
=
\{\mu_{y,k}\le0\}
\cap
\left\{
\pi(n_+,M)
\le
\rho\delta_{t,p}
\right\},
\]
and the confidence-bound failure event
\[
E^{\mathrm{cert}}_{t,y,k}
=
\left\{
\mu_{y,k}
<
\ell^\star
\left(
Z_{1:n};
(1-\rho)\delta_{t,p}
\right)
\right\}.
\]
By \cref{lem:sign},
\[
\Prob(E^{\mathrm{dec}}_{t,y,k})
\le
\rho\delta_{t,p},
\]
and by \cref{lem:signlcb},
\[
\Prob(E^{\mathrm{cert}}_{t,y,k})
\le
(1-\rho)\delta_{t,p}.
\]
Set $E_{t,y,k}=E^{\mathrm{dec}}_{t,y,k}\cup E^{\mathrm{cert}}_{t,y,k}$. Then
\[
\Prob(E_{t,y,k})\le\delta_{t,p}.
\]

For the bounded-score extension, define instead
\[
E_{t,y,k}
=
\left\{
\mu_{y,k}
<
\mathrm{lcb}(Z_{1:n};\cL_{t,p})
\right\}.
\]
By \cref{lem:bkl},
\[
\Prob(E_{t,y,k})
\le
\delta_{t,p}.
\]

Now sum over every round, promotion count, decision history, and candidate. Using
\eqref{eq:alloc} and \eqref{eq:history-count},
\begin{align}
\Prob\left(
\bigcup_{t,p,y,k}E_{t,y,k}
\right)
&\le
\sum_{t\ge1}
\sum_{p=0}^{t-1}
\binom{t-1}{p}K^p
\cdot K
\cdot
\delta_{t,p}
\\
&=
\alpha
\sum_{t\ge1}w_t
\sum_{p=0}^{t-1}v_p
\\
&\le
\alpha.
\end{align}

Therefore, with probability at least $1-\alpha$, every candidate comparison that can arise under every possible
decision history is valid simultaneously.

Consider the realized sequence. If $y_t=k\neq0$, the promoted candidate passed the promotion test. On the event above,
with $\gamma=0$ for the sign-test route and any $\gamma\ge0$ for the bounded-score route,
\[
\Delta(A_t,A_{t-1})>\gamma,
\]
which proves the first claim.

The same event also gives
\[
\Delta(A_t,A_{t-1})
\ge
\max\{\ell_t,\gamma\}
\]
at every promoted round, where $\ell_t$ denotes the lower bound used for the promoted comparison.
Summing over promoted rounds gives
\begin{align}
V(A_t)-V(A_0)
&=
\sum_{s\le t:y_s\neq0}
\Delta(A_s,A_{s-1})
\\
&\ge
\sum_{s\le t:y_s\neq0}
\max\{\ell_s,\gamma\}
\\
&=
C_t.
\end{align}

{
This proves the cumulative lower bound for each fixed $\xi=(R,D,A_0)$.
By the joint measurability in \cref{ass:iso}, the failure events used
above, for each fixed hypothetical history, are measurable as functions
of $(\xi,S)$. Their union over rounds and histories is a countable
union of measurable sets. Since $S$ and $\xi$ are independent by
\cref{ass:indep}, integrating the section-wise probability bound
with respect to the law of $\xi$ gives the unconditional statement.}

\end{proof}



\subsection{A direct lower bound against the starting system}
\label{app:direct}

The running bound $C_t$ adds one lower confidence bound for each promoted modification. We can also compare the
current system $A_t$ directly with the starting system $A_0$. This produces a second lower bound on total progress
from the same stored evaluations.

Let $\alpha_c\in(0,1)$ be a separate error level. For a decision history of length $t$ containing $p$ promotions,
define
\begin{equation}
\label{eq:direct-allocation}
\delta^c_{t,p}
=
\frac{
\alpha_c w_t v_p
}{
\binom{t}{p}K^p
},
\qquad
\cL^c_{t,p}
=
\ln\frac{2}{\delta^c_{t,p}}.
\end{equation}

When $p_t\ge1$, define
\begin{equation}
\label{eq:direct-certificate}
D_t
=
\ell^\star
\left(
Z_{1:n}(A_t,A_0);
\delta^c_{t,p_t}
\right),
\end{equation}
for pass/fail outcomes. For general bounded scores, use
$D_t=\mathrm{lcb}(Z_{1:n}(A_t,A_0);\cL^c_{t,p_t})$ from \cref{lem:bkl} instead.
If $p_t\ge1$ and $\delta^c_{t,p_t}=0$, set $D_t=-\infty$. When $p_t=0$, set
\[
D_t=0,
\]
since $A_t=A_0$ exactly.

\begin{proposition}[Direct lower bound on total progress]
\label{prop:direct}
Under \cref{ass:indep,ass:iso}, with probability at least $1-\alpha_c$,
\[
V(A_t)-V(A_0)
\ge
D_t
\qquad
\text{for all }t\ge1.
\]
Under development-set screening, the same result holds with the factor $K^p$ removed from
\eqref{eq:direct-allocation}. 
{In the pass/fail setting with $\gamma=0$, under the conditions
of \cref{thm:main}, we additionally have}
\[
V(A_t)-V(A_0)
\ge
\max\{C_t,D_t\}
\]
for all $t$ with probability at least $1-\alpha-\alpha_c$.
\end{proposition}

\begin{proof}
Condition on $(R,D,A_0)$ as in the proof of \cref{thm:main}. Probabilities below are conditional on this tuple.

A complete decision history
\[
y=y_{1:t}
\]
determines the current system $a_t(y)$. If the history contains exactly $p$ promotions, then there are
\[
\binom{t}{p}K^p
\]
such histories.

For each history with $p\ge1$, define
\[
F_{t,y}
=
\left\{
\Delta(a_t(y),A_0)
<
\ell^\star
\left(
Z_{1:n}(a_t(y),A_0);
\delta^c_{t,p}
\right)
\right\}.
\]
For each fixed hypothetical history $y$, the pair $(a_t(y),A_0)$ is fixed given $(R,D,A_0)$ and does not otherwise
depend on $S$. We do not condition on observing $y$. Thus
\[
\Prob(F_{t,y})
\le
\delta^c_{t,p}.
\]

For bounded scores, replace $\ell^\star$ by the bound from \cref{lem:bkl}; the same error bound holds.
When $p=0$, no promotion has occurred and $a_t(y)=A_0$, so the desired inequality holds deterministically;
set $F_{t,y}=\varnothing$ for these histories. The failure event is also empty when the error allotment is zero.

A union bound gives
\begin{align}
\Prob\left(
\bigcup_{t,p,y}F_{t,y}
\right)
&\le
\sum_{t\ge1}
\sum_{p=0}^{t}
\binom{t}{p}K^p
\delta^c_{t,p}
\\
&=
\alpha_c
\sum_{t\ge1}w_t
\sum_{p=0}^{t}v_p
\\
&\le
\alpha_c.
\end{align}
The result follows for the realized decision history. Integrating over $(R,D,A_0)$ gives the unconditional statement.

Under development screening, a history is determined only by its promotion rounds, giving $\binom{t}{p}$ possible
histories and removing $K^p$. The joint statement follows from a union bound over the events of
\cref{thm:main} and this proposition.
\end{proof}

\begin{corollary}[Carry-forward between promotions]
\label{cor:carry}
Suppose $p_t\ge1$ and let
\[
s(t)=\max\{s\le t:y_s\neq0\}
\]
be the most recent promotion round. On the event of \cref{prop:direct},
\[
V(A_t)-V(A_0)
\ge
D_{s(t)}.
\]
Hence the direct lower bound only needs to be recomputed when a promotion occurs.
\end{corollary}

\begin{proof}
No promotion occurs between rounds $s(t)$ and $t$, so
\[
A_t=A_{s(t)}.
\]
Apply \cref{prop:direct} at round $s(t)$.
\end{proof}

\paragraph{Comparison with the running bound.}
The two bounds summarize cumulative progress differently. The running quantity $C_t$ adds one confidence bound for
each promoted step. The direct quantity $D_t$ bounds the single comparison between $A_t$ and $A_0$. When several
promotions accumulate, $D_t$ can be substantially tighter because it pays one confidence margin rather than one
margin for every promotion.

Neither bound dominates in every configuration, so both can be reported under their respective error levels.


\subsection{Power and the minimum detectable improvement}
\label{app:power}

We next characterize the amount of population gain required to pass the promotion test for fixed comparisons.
The pairs below are fixed independently of $S$. Equivalently, one may fix $(R,D,A_0)$ and a hypothetical decision
history as in the proof of \cref{thm:main}. These power statements are not conditional on reaching that history
in an adaptive run.

\begin{proposition}[Power of the promotion tests]
\label{prop:powerfull}
Consider a fixed candidate--incumbent pair tested at round $t$ after $p$ promotions, with i.i.d.\ paired
differences $Z_1,\ldots,Z_n$, population gain $\Delta$, paired second moment
\[
m=\E Z^2,
\]
and variance $\sigma^2$. Let
\[
L=\cL_{t,p},
\qquad
\beta\in(0,1).
\]

\begin{enumerate}
\item[(a)]
For the Bernstein--KL test, define
\[
L_\beta=\ln\frac2\beta,
\]
\[
m_\beta
=
m+
\sqrt{\frac{2mL_\beta}{n}}
+
\frac{2L_\beta}{3n},
\]
and
\[
u
=
m_\beta
+
\sqrt{\frac{2m_\beta L}{n}}
+
\frac{2L}{n}.
\]
If
\[
\Delta
>
\gamma
+
\varepsilon(u,L)
+
\varepsilon(\sigma^2,L_\beta),
\]
then this candidate passes the bounded-score promotion test with probability at least $1-\beta$.

For pass/fail tasks with fixed $d=\Prob(Z_i\neq0)>0$, the right-hand side of the sufficient condition above is
\[
\gamma+\sqrt{\frac{2dL}{n}}(1+o(1))
\]
when $n\to\infty$, $L\to\infty$, $L=o(n)$, and $L_\beta=o(L)$.

\item[(b)]
For the exact paired sign test with $\gamma=0$ and $\rho=1/2$, define
\[
c(M,L)
=
\min
\left\{
a\in\{0,1,\ldots,M\}:
\Prob\!\left(
\mathrm{Bin}(M,\tfrac12)\ge a
\right)
\le e^{-L}
\right\},
\]
with $c(M,L)=\infty$ when the set is empty.

A candidate passes if and only if
\[
n_+\ge c(M,L).
\]
For $M\ge1$ and $n_+\ge M/2$,
\begin{equation}
\label{eq:signsandwich}
M
\kl\left(
\frac{n_+}{M}
\,\middle\|\,
\frac12
\right)
\ge L
\quad
\Longrightarrow
\quad
\text{pass}
\quad
\Longrightarrow
\quad
M
\kl\left(
\frac{n_+}{M}
\,\middle\|\,
\frac12
\right)
\ge
L-\frac12\ln(2M).
\end{equation}
In particular, for $M\ge1$,
\[
n_+-n_-
\ge
\sqrt{2LM}
\]
is sufficient for promotion.
\end{enumerate}
\end{proposition}

\begin{proof}

{For part (a), apply Bernstein's inequality
\citep{hoeffding1963} to $Y_i=Z_i^2\in[0,1]$.
Here $Y_i-m\le1$, so the range parameter is $b=1$, and
\[
\mathrm{Var}(Y_i)\le\E Z_i^4\le\E Z_i^2=m.
\]
Inverting the resulting tail bound, and using
$\sqrt{u+v}\le\sqrt u+\sqrt v$ for $u,v\ge0$, gives
\[
\overline{Z^2}
\le m+\sqrt{\frac{2mL_\beta}{n}}+\frac{2L_\beta}{3n}
=m_\beta
\]
with probability at least $1-\beta/2$.
Applying the same Bernstein inequality to $-Z_i$, now with $b=2$
and variance $\sigma^2$, gives
\[
\bar Z\ge\Delta-\varepsilon(\sigma^2,L_\beta)
\]
with probability at least $1-\beta/2$.

The inequality}

\[
\kl(a\,\|\,q)
\ge
\frac{(q-a)^2}{2q},
\qquad
q\ge a,
\]
together with the definition of $\hat s^+$ implies
\[
\hat s^+
\le
\overline{Z^2}
+
\sqrt{\frac{2\overline{Z^2}L}{n}}
+
\frac{2L}{n}
\le u.
\]
On the intersection of the two high-probability events,
\[
\ell
\ge
\Delta
-
\varepsilon(\sigma^2,L_\beta)
-
\varepsilon(u,L)
>
\gamma.
\]
Hence the candidate passes. The asymptotic expression follows from
\[
\varepsilon(u,L)
=
\sqrt{\frac{2mL}{n}}(1+o(1)).
\]

For part (b), the test level is
\[
\rho\delta_{t,p}
=
\frac{\delta_{t,p}}2
=
e^{-L}.
\]
Monotonicity of the binomial upper tail gives the characterization through $c(M,L)$.
When $M=0$, $c(M,L)=\infty$ and the test cannot pass; the following calculations assume $M\ge1$.

{For $n_+\ge M/2$, the Chernoff--Hoeffding bound
\citep{hoeffding1963} gives
\[
\Prob\!\left(\mathrm{Bin}(M,\tfrac12)\ge n_+\right)
\le
\exp\!\left\{-M\kl\!\left(\frac{n_+}{M}\,\middle\|\,\frac12\right)\right\}.
\]
The boundary $n_+=M/2$ follows from the trivial probability bound by one,
and $n_+=M$ gives equality. This proves the first implication in
\eqref{eq:signsandwich}.

The binomial point-mass bound \eqref{eq:binom-half-mass-lb} gives, for every integer $0\le a\le M$,
\[
\Prob\!\left(\mathrm{Bin}(M,\tfrac12)=a\right)
\ge\frac{\exp\{-M\kl(a/M\,\|\,1/2)\}}{\sqrt{2M}}.
\]
If the candidate passes, then
\[
\frac{\exp\{-M\kl(n_+/M\,\|\,1/2)\}}{\sqrt{2M}}
\le\Prob\!\left(\mathrm{Bin}(M,\tfrac12)=n_+\right)
\le\pi(n_+,M)\le e^{-L}.
\]
Taking logarithms yields
$M\kl(n_+/M\,\|\,1/2)\ge L-\frac12\ln(2M)$,
which is the second implication in \eqref{eq:signsandwich}.}

Finally, write
\[
n_+
=
\frac{M+x}{2},
\qquad
x=n_+-n_-.
\]
For $u=x/(2M)$ with $|u|<1/2$,
\begin{equation}
\label{eq:klseries}
2u^2
\le
\kl\left(
\frac12+u
\,\middle\|\,
\frac12
\right)
\le
\frac{2u^2}{1-4u^2}.
\end{equation}
The lower inequality extends to $u=\pm1/2$ by continuity and gives
\[
M\kl\left(
\frac{n_+}{M}
\,\middle\|\,
\frac12
\right)
\ge
\frac{x^2}{2M}.
\]
Thus $x\ge\sqrt{2LM}$ is sufficient for promotion.
\end{proof}

\paragraph{Asymptotic conditions for \cref{prop:power}.}
Let $n\to\infty$, keep $K$ and $d\in(0,1]$ fixed, and set $L=\cL_{t,p}$ with
$\ln n=o(L)$ and $L=o(n)$. The round and promotion count may vary with $n$.
All candidate--incumbent pairs are fixed independently of $S$; in the $K$-candidate statement,
each pair has the same probability $d$ of different outcomes. Candidates need not be independent of one another.

\begin{proof}[Proof of \cref{prop:power}]
Let
\[
x
=
\sum_{i=1}^n Z_i
=
n_+-n_-,
\qquad
M
=
\sum_{i=1}^n Z_i^2.
\]
For pass/fail outcomes,
\[
\E x=n\Delta,
\qquad
\mathrm{Var}(x)\le nd,
\]
and
\[
\E M=nd,
\qquad
\mathrm{Var}(M)\le nd.
\]
Hence
\[
x
=
n\Delta+O_P(\sqrt{nd}),
\]
and
\[
M
=
nd
\left(
1+O_P((nd)^{-1/2})
\right).
\]

Let
\[
\tau=\sqrt{2Lnd}.
\]
Then
\[
\Delta_{\min}
=
\sqrt{\frac{2dL}{n}}
=
\frac{\tau}{n}.
\]
Since $L\to\infty$,
\[
\frac{\tau}{\sqrt{nd}}
=
\sqrt{2L}
\to\infty.
\]
Therefore the stochastic fluctuations in $x$ and $M$ are $o_P(\tau)$.

If
\[
\Delta
\ge
(1+\eta)\Delta_{\min},
\]
then
\[
x
\ge
(1+\eta)\tau-o_P(\tau),
\]
while
\[
\sqrt{2LM}
=
\tau(1+o_P(1)).
\]
Thus the sufficient condition in \cref{prop:powerfull}(b) holds with probability tending to one.

If
\[
\Delta
\le
(1-\eta)\Delta_{\min},
\]
then
\[
x
\le
(1-\eta)\tau+o_P(\tau).
\]
On $\{x\le0\}$ the test cannot pass. On $\{x>0\}$, the preceding upper bound on $x$, together with
$M/(nd)\to_P1$ and $L=o(n)$, gives $x/M=o_P(1)$ on this event.
Using \eqref{eq:klseries},
\[
M
\kl\left(
\frac{n_+}{M}
\,\middle\|\,
\frac12
\right)
=
\frac{x^2}{2M}(1+o_P(1)).
\]
A pass would require
\[
x^2
\ge
2M
\left(
L-\frac12\ln(2M)
\right)
(1+o_P(1))
=
\tau^2(1+o_P(1)),
\]
where
\[
\frac{\ln(2M)}{L}\to0
\]
by $\ln n=o(L)$. This contradicts
\[
x^2
\le
\left(
1-\eta+o_P(1)
\right)^2
\tau^2
\]
with probability tending to one.

For fixed $K$, a union bound over the fixed candidate--incumbent pairs completes the second claim.
\end{proof}


\subsection{Lower bound for threshold-based promotion rules}
\label{lb:sec}
\label{app:lb}

We now show that the dependence on promotion history identified in the main text also appears as a lower bound for
a natural class of promotion rules. The construction uses a bounded-score model satisfying
\cref{ass:indep,ass:iso} exactly.

\subsubsection{Model and threshold gates}

\paragraph{Model $\mathsf M$.}
A task is an infinite sign sequence
\[
W=(g_1,g_2,\ldots)
\in
\{-1,+1\}^{\mathbb N},
\]
whose coordinates are independent Rademacher variables.

A system is a pair
\[
A=(\vartheta,\beta),
\]
where
\[
\vartheta\in[-1/4,1/4],
\qquad
\beta\in\mathbb R^{\mathbb N},
\qquad
\|\beta\|_1\le1/4,
\]
and $\beta$ has finite support. Define
\begin{equation}
\label{lb:eq-score}
s(A,W)
=
\frac12+\vartheta+\sum_j\beta_jg_j.
\end{equation}
The score lies in $[0,1]$ and
\[
V(A)
=
\frac12+\vartheta.
\]
Hence
\[
\Delta(B,A)
=
\vartheta_B-\vartheta_A.
\]
Changing only $\beta$ leaves population value unchanged. All constructions below start from $A_0=(0,0)$.
We write $e_j$ for the $j$-th coordinate vector.

Let
\[
S=(W_1,\ldots,W_n)
\]
be the evaluation set. Write $g_{ij}$ for coordinate $j$ of task $i$, and define
\[
X_j
=
n^{-1/2}
\sum_{i=1}^n g_{ij}.
\]
The variables $X_j$ are independent standardized Rademacher sums.

For any systems $A$ and $B$,
\begin{equation}
\label{lb:eq-Z}
Z_i(B,A)
=
(\vartheta_B-\vartheta_A)
+
\sum_j
(\beta_B-\beta_A)_j
g_{ij}.
\end{equation}
Coordinates shared by $A$ and $B$ cancel. There is no additional evaluation randomness, so repeated evaluation is
deterministic.

\begin{definition}[Threshold gate]
\label{lb:def-gate}
Fix $b\in[0,1]$. For $z\in\mathbb R^n$, define
\[
\bar z
=
\frac1n\sum_i z_i,
\qquad
\overline{z^2}
=
\frac1n\sum_i z_i^2,
\]
and
\begin{equation}
\label{lb:eq-stat}
T_b(z)
=
\frac{
\sqrt n\,\bar z
}{
\sqrt{
\overline{z^2}
-
b\bar z^2
}
}.
\end{equation}
We set $T_b(0)=0$. When the denominator vanishes and $\bar z\neq0$, set
$T_b(z)=+\infty$ for $\bar z>0$ and $T_b(z)=-\infty$ for $\bar z<0$.

A threshold gate is specified by $b$ and deterministic thresholds
\[
\lambda_{t,p}\ge0.
\]
At round $t$ after $p$ previous promotions, it computes
\[
T_k
=
T_b
\left(
Z_{1:n}(B_{t,k},A_{t-1})
\right).
\]
It rejects all candidates when
\[
\max_k T_k
\le
\sqrt{2\lambda_{t,p}},
\]
and otherwise promotes a candidate attaining the maximum.
\end{definition}

The case $b=0$ standardizes by the empirical second moment. The case $b=1$ uses the empirical variance with divisor
$n$.

All proposers below satisfy \cref{ass:indep,ass:iso}. They are deterministic and use only previous promotion
decisions. We set $\gamma=0$. Let
\[
\mathfrak F_t
=
\{
\text{a false promotion occurs by round }t
\}.
\]

\begin{theorem}[First-order lower bound]
\label{thm:lb}
For the threshold-gate class in \cref{lb:def-gate}, under the conditions of
\cref{lb:cor-a,lb:cor-rate} and the corresponding asymptotic regimes in
\cref{lb:rem-ladder}, a proposer that observes only previous promotion
decisions can cause a false promotion with probability tending to one
if the threshold parameters are uniformly below the first-order scale
$\ln(tK)$ before the first promotion, or $P\ln(Kt/P)$ over the relevant
rounds after $P$ promotions, by a fixed multiplicative factor.
In the latter regime, the conditions include
$t=\lfloor\varsigma P\rfloor+1$, $\varsigma\ge5$,
$P\to\infty$, $K\varsigma\to\infty$,
$\ln(K\varsigma)=o(\sqrt n)$, and $P\ln(eK\varsigma)=o(n)$,
as well as $\lambda_{s,p}<n/34$ for every $1\le s\le t-1$ and
$0\le p\le\min\{P-1,s-1\}$, as required by \cref{lb:cor-rate}.
\end{theorem}

The two regimes are established separately in \cref{lb:cor-a,lb:cor-rate}; the remaining subsections
give the corresponding constructions and proofs.


\subsubsection{Auxiliary lemmas}

{
\paragraph{Binomial coefficient and point-mass bounds.}
For integers $n\ge2$ and $1\le m\le n-1$,
\begin{equation}
\label{eq:binom-entropy-lb}
\binom nm
\ge\frac{\exp\{n\ln2-n\kl(m/n\,\|\,1/2)\}}
{\sqrt{8m(n-m)/n}}.
\end{equation}
Consequently, for integers $M\ge1$ and $0\le a\le M$,
\begin{equation}
\label{eq:binom-half-mass-lb}
\Prob\!\left(\mathrm{Bin}(M,\tfrac12)=a\right)
\ge\frac{\exp\{-M\kl(a/M\,\|\,1/2)\}}{\sqrt{2M}}.
\end{equation}

\begin{proof}
For integers $j\ge1$, define
\[
A_j=\frac{j!e^j}{j^{j+1/2}},
\qquad
r_j=\frac{A_{j+1}}{A_j}
=e\left(\frac{j}{j+1}\right)^{j+1/2}.
\]
The function $g(x)=(x+1/2)\ln(1+1/x)$ satisfies, for $x>0$,
\[
g'(x)=\ln(1+1/x)-\frac12\left(\frac1x+\frac1{x+1}\right)<0.
\]
Indeed, strict convexity of $t\mapsto1/t$ gives
$\int_x^{x+1}t^{-1}\,dt<\frac12(x^{-1}+(x+1)^{-1})$.
Thus $\ln r_j=1-g(j)$ is increasing in $j$.
For positive integers $m,l$, set $F(m,l)=A_{m+l}/(A_mA_l)$.
Then
\[
\frac{F(m+1,l)}{F(m,l)}=\frac{r_{m+l}}{r_m}\ge1,
\]
and the analogous inequality holds in the second coordinate. Hence
\[
F(m,l)\ge F(1,1)=\frac{A_2}{A_1^2}=\frac1{\sqrt8}.
\]
Writing $n=m+l$ and expanding the factorials therefore gives
\[
\binom nm
=F(m,l)\sqrt{\frac{n}{ml}}\frac{n^n}{m^ml^l}
\ge\frac{\exp\{nH(m/n)\}}{\sqrt{8ml/n}},
\]
where $H(x)=-x\ln x-(1-x)\ln(1-x)$.
Since $H(x)=\ln2-\kl(x\,\|\,1/2)$, this proves
\eqref{eq:binom-entropy-lb}.

For $1\le a\le M-1$, divide \eqref{eq:binom-entropy-lb} by $2^M$
and use $a(M-a)\le M^2/4$ to obtain \eqref{eq:binom-half-mass-lb}.
For $a=0$ or $a=M$, both the point mass and
$\exp\{-M\kl(a/M\,\|\,1/2)\}$ equal $2^{-M}$, and
$\sqrt{2M}\ge1$, so the bound still holds.
\end{proof}
}

\begin{lemma}[Rademacher tail]
\label{lb:lem-tail}
Let $S_n$ be a sum of $n$ independent Rademacher variables. For $n\ge64$ and
\[
1\le x\le\frac{\sqrt n}{8},
\]
\[
\Prob(S_n>x\sqrt n)
\ge
\frac{c_R}{x}
\exp\left(
-\frac{x^2}{2}
-
\frac{(x^2+5)^2}{9n}
\right),
\]
where
\[
c_R=\frac{e^{-5/2}}{2\sqrt2}.
\]
\end{lemma}

\begin{proof}
Let
\[
B=\frac{S_n+n}{2}
\sim
\mathrm{Bin}(n,\tfrac12).
\]
Set
\[
m_0
=
\left\lfloor
\frac{n+x\sqrt n}{2}
\right\rfloor+1,
\qquad
w
=
\left\lfloor
\frac{\sqrt n}{2x}
\right\rfloor,
\qquad
m_1=m_0+w.
\]
Then
\[
\{S_n>x\sqrt n\}
=
\{B\ge m_0\}.
\]

For $m\ge n/2$, the binomial probability mass is nonincreasing in $m$, so
\[
\Prob(B\ge m_0)
\ge
(w+1)\Prob(B=m_1)
\ge
\frac{\sqrt n}{2x}\Prob(B=m_1).
\]

{Applying \eqref{eq:binom-entropy-lb} with $m=m_1$, which satisfies
$1\le m_1\le n-1$, and using $m_1(n-m_1)\le n^2/4$, we obtain
\[
\Prob(B=m_1)
\ge
\frac{
e^{-n\kl(m_1/n\,\|\,1/2)}
}{
\sqrt{2n}
}.
\]
}

Writing $m_1/n=(1+v)/2$ and $u=\sqrt n\,v$, the assumptions imply
\[
x<u\le x+\frac2{\sqrt n}+\frac1x
\]
and
\[
u^2\le x^2+5.
\]
The expansion of the Bernoulli relative entropy gives
\[
n
\kl\left(
\frac{1+v}{2}
\,\middle\|\,
\frac12
\right)
\le
\frac{x^2+5}{2}
+
\frac{(x^2+5)^2}{9n}.
\]
Combining the inequalities gives the result.
\end{proof}

\begin{lemma}[Uniform second moment of selected probes]
\label{lb:lem-subset}
Let $n,M,P$ be positive integers with $P\le M$, and let $(g_{ij})_{i\le n,j\le M}$ be independent Rademacher variables. For
$J\subseteq[M]$ with $|J|=P$, define
\[
Q_J
=
\frac1{nP}
\sum_{i=1}^n
\left(
\sum_{j\in J}g_{ij}
\right)^2.
\]
For every $\varepsilon>0$,
\[
\Prob\left(
\max_{|J|=P}Q_J
\ge
1+\varepsilon
\right)
\le
\binom MP
\exp\left\{
-\frac n2
\left[
\varepsilon-\ln(1+\varepsilon)
\right]
\right\},
\]
and therefore
\[
\Prob\left(
\max_{|J|=P}Q_J
\ge
1+\varepsilon
\right)
\le
\left(
\frac{eM}{P}
\right)^P
\exp\left\{
-\frac{n\varepsilon^2}{4(1+\varepsilon)}
\right\}.
\]
\end{lemma}

\begin{proof}
Fix $J$ and set
\[
Y_i
=
P^{-1/2}
\sum_{j\in J}g_{ij}.
\]
Then
\[
\E e^{\mu Y_i}
=
\cosh(\mu/\sqrt P)^P
\le
e^{\mu^2/2}.
\]
If $\xi\sim N(0,1)$ is independent, then for $0\le s<1/2$,
\[
\E e^{sY_i^2}
=
\E e^{\sqrt{2s}\xi Y_i}
\le
\E e^{s\xi^2}
=
(1-2s)^{-1/2}.
\]
Markov's inequality with
\[
s=\frac{\varepsilon}{2(1+\varepsilon)}
\]
gives
\[
\Prob\left(
\sum_iY_i^2
\ge
n(1+\varepsilon)
\right)
\le
\exp\left\{
-\frac n2
\left[
\varepsilon-\ln(1+\varepsilon)
\right]
\right\}.
\]
Apply a union bound over the $\binom MP$ possible sets $J$ and use
\[
\binom MP\le(eM/P)^P
\]
and
\[
\varepsilon-\ln(1+\varepsilon)
\ge
\frac{\varepsilon^2}{2(1+\varepsilon)}.
\]
\end{proof}

\begin{lemma}[Calibration of a probe round]
\label{lb:lem-calib}

Let $n\ge3$, $b\in[0,1]$, and $\tau>0$. For
$r\in[0,1/4]$ and $x\in[-1,1]$, define
\[
h_r(x)
=
\frac{
\sqrt n(r+x)
}{
\sqrt{D_r(x)}
},
\qquad
D_r(x)
=
1+r^2+2rx-b(r+x)^2.
\]
Then the following properties hold.

\begin{enumerate}
\item
If $g\in\{-1,1\}^n$ has empirical mean $x$, then
\[
T_b\bigl(\tau(r\mathbf1+g)\bigr)
=
h_r(x).
\]

\item
For fixed $r$, $h_r(x)$ is strictly increasing in $x$. For $x\in(-1,1)$, it is continuous and strictly increasing
in $r$.

\item
For $x\in[0,1]$,
\[
h_{1/4}(x)
\ge
\frac{\sqrt n}{\sqrt{17}}.
\]

\item
Let
\[
\Lambda_n
=
\{-1,-1+2/n,\ldots,1\}.
\]
For $u\in\Lambda_n\cap(0,1)$ and $u^-=u-2/n$, every threshold
\[
h_0(u)
\le
\theta
<
h_{1/4}(u)
\]
admits an $r\in(0,1/4]$ such that
\[
h_r(x)>\theta
\quad\Longleftrightarrow\quad
x\ge u
\]
for every $x\in\Lambda_n$.
\end{enumerate}
\end{lemma}

\begin{proof}
For $z=\tau(r\mathbf1+g)$,
\[
\bar z=\tau(r+x),
\qquad
\overline{z^2}
=
\tau^2(1+r^2+2rx),
\]
which proves the first claim.

Direct differentiation gives
\[
\partial_xh_r(x)
=
\sqrt n(1+rx)D_r(x)^{-3/2}>0
\]
and
\[
\partial_rh_r(x)
=
\sqrt n(1-x^2)D_r(x)^{-3/2}>0
\]
for $|x|<1$. This proves the monotonicity statements.

The third claim follows from
\[
h_{1/4}(x)
\ge
h_{1/4}(0)
=
\frac{
\sqrt n
}{
4\sqrt{1+(1-b)/16}
}
\ge
\frac{\sqrt n}{\sqrt{17}}.
\]

For the final claim, let
\[
f(r)=h_r(u),
\qquad
\varphi(r)=h_r(u^-).
\]
Both functions are continuous and strictly increasing and satisfy
\[
\varphi(r)<f(r).
\]
The intermediate value argument therefore produces an $r$ for which
\[
\varphi(r)<\theta<f(r).
\]
Monotonicity in $x$ gives the desired cutoff.
\end{proof}


\subsubsection{The cost of repeated candidate testing}

\begin{theorem}[Neutral candidates defeat thresholds below $\ln(tK)$]
\label{lb:thm-a}
In model $\mathsf M$, let $(b,\lambda)$ be a threshold gate. Let $n\ge64$, $K,t\ge1$, and define
\[
N=tK,
\qquad
\Lambda_t
=
\max_{s\le t}\lambda_{s,0},
\qquad
x
=
\max\{1,\sqrt{2\Lambda_t}\}.
\]
Assume $x\le\sqrt n/8$.

Consider the non-adaptive proposer that submits
\[
B_{s,k}
=
\left(
0,
\frac14 e_{(s-1)K+k}
\right).
\]
Every candidate has zero population gain. If
\[
q=\Prob(S_n>x\sqrt n),
\]
then
\[
\Prob(\mathfrak F_t^c)
\le
(1-q)^N
\le
\exp\left\{
-\frac{c_RN}{x}
e^{-x^2/2-(x^2+5)^2/(9n)}
\right\}.
\]
\end{theorem}

\begin{proof}
Until a promotion occurs, the incumbent remains $A_0$. Candidate $(s,k)$ depends on a fresh probe coordinate
$j=(s-1)K+k$. By \cref{lb:lem-calib},
\[
T_k
\ge
X_j
\]
whenever $X_j\ge0$.

Since
\[
\sqrt{2\lambda_{s,0}}
\le
\sqrt{2\Lambda_t}
\le x,
\]
any probe with $X_j>x$ forces a promotion. All such promotions are false because every candidate is population
neutral.

The $N=tK$ probe coordinates are independent. Hence
\[
\Prob(\mathfrak F_t^c)
\le
(1-q)^N
\le
e^{-Nq}.
\]
Apply \cref{lb:lem-tail}.
\end{proof}

\begin{corollary}
\label{lb:cor-a}
In the setting of \cref{lb:thm-a}, suppose
\[
\Lambda_t
\le
(1-\eta)\ln N
\]
for a fixed $\eta\in(0,1]$, with $N\ge2$ and $\ln N\le n/128$. Then
\[
\Prob(\mathfrak F_t)
\ge
1-
\exp\left\{
-c_Re^{-1/2}
\frac{
N^\eta
}{
\sqrt{1+2\ln N}
}
e^{-(2\ln N+6)^2/(9n)}
\right\}.
\]
In particular,
\[
\Prob(\mathfrak F_t)\to1
\]
whenever
\[
tK\to\infty,
\qquad
\frac{\ln(tK)}{n}\to0.
\]
\end{corollary}

\begin{proof}
Apply \cref{lb:thm-a} with
\[
x
\le
\max\{1,\sqrt{2\ln N}\}.
\]
The stated bound follows by substituting the tail estimate from \cref{lb:lem-tail}. Under the asymptotic conditions,
the exponent diverges.
\end{proof}

\begin{remark}
\label{lb:rem-a-rounds}
The result depends on the thresholds used throughout the tested rounds. More generally, if a set of rounds
$\mathcal R\subseteq[t]$ satisfies
\[
\lambda_{s,0}
\le
(1-\eta)\ln(|\mathcal R|K)
\qquad
\text{for }s\in\mathcal R,
\]
the proposer can submit neutral probes only in those rounds. The same argument then applies with
\[
N=|\mathcal R|K.
\]
\end{remark}


\subsubsection{The cost of promotion history}

Fix
\[
a\in(0,\sqrt n/8].
\]
Let
\[
\ell
=
\min\{v\in\sqrt n\,\Lambda_n:v\ge a\},
\qquad
u=\ell/\sqrt n,
\]
and define
\[
q
=
\Prob(S_n\ge a\sqrt n),
\qquad
\pi
=
1-(1-q)^K.
\]
For a target number $P$ of promotions and $\eta'>0$, set
\[
t_0
=
\left\lceil
\frac{(1+\eta')P}{\pi}
\right\rceil,
\qquad
M=Kt_0.
\]

We construct a proposer that uses previous promotion decisions to identify probe directions favored by the evaluation
set. In a shifted-probe round, a successful probe is attached to a genuine positive shift. A neutral-probe promotion
is already false. If no false promotion has occurred, after $P$ genuine promotions the proposer combines the selected
neutral probes into one final candidate whose population gain is zero.

Let
\[
\tau=\frac{1}{8P}.
\]
The proposer keeps the selected probe indices in a set $J$, initially empty, and uses fresh indices
$j_{s,k}=(s-1)K+k$. At round $s$, let $p=|J|$ and
\[
\theta_{s,p}
=
\sqrt{2\lambda_{s,p}}.
\]

\paragraph{Adaptive proposer $\mathcal P_B$.}
The proposer acts as follows.
\begin{enumerate}
\item
If a false promotion has already occurred or the final candidate has already been submitted, submit copies of the
incumbent.

\item
If $p=P$, submit $K$ copies of
\[
B^{\mathrm{fin}}
=
A_{s-1}
+
\left(
0,
\tau\sum_{j\in J}e_j
\right).
\]

\item
If
\[
\theta_{s,p}<h_0(u),
\]
submit neutral probes
\[
B_{s,k}
=
A_{s-1}
+
(0,\tau e_{j_{s,k}}).
\]

\item
If
\[
h_0(u)
\le
\theta_{s,p}
<
h_{1/4}(u),
\]
choose $r_s$ using \cref{lb:lem-calib} and submit
\[
B_{s,k}
=
A_{s-1}
+
(r_s\tau,\tau e_{j_{s,k}}).
\]

\item
Otherwise submit copies of the incumbent.
\end{enumerate}

If a candidate is promoted in a probe round, its probe index is added to $J$. The proposer can identify a false
promotion from its own construction: neutral probes and the final candidate have zero population gain.
Thus the construction depends only on previous promotion decisions and satisfies \cref{ass:iso}.
Every submitted system remains in $\mathsf M$: before the final round, $\|\beta\|_1\le P\tau=1/8$ and
$0\le\vartheta\le P\tau/4=1/32$, while the final candidate has $\|\beta\|_1\le2P\tau=1/4$.

\begin{theorem}[Adaptive proposer defeats slowly growing thresholds]
\label{lb:thm-b}
In model $\mathsf M$, let $(b,\lambda)$ be a threshold gate. With the parameters defined above, fix $\varepsilon>0$ and assume
\begin{enumerate}
\item[(B1)]
\[
\lambda_{s,p}
<
\frac{n}{34}
\]
for all $s\le t_0$ and $p\le\min\{P-1,s-1\}$;

\item[(B2)]
\[
\lambda_{s,P}
\le
\frac{a^2P}{2(1+\varepsilon)}
\]
for all $P+1\le s\le t_0+1$.
\end{enumerate}
Then the proposer $\mathcal P_B$ satisfies
\[
\Prob(\mathfrak F_{t_0+1}^c)
\le
\exp\left\{
-\frac{\eta'^2P}{2(1+\eta')}
\right\}
+
\binom MP
\exp\left\{
-\frac n2
\bigl(
\varepsilon-\ln(1+\varepsilon)
\bigr)
\right\}.
\]
\end{theorem}

\begin{proof}

{
First suppose $n\in\{1,2\}$. Since $0<a\le\sqrt n/8$, the smallest
positive value of a standardized Rademacher sum is $\sqrt n$, so
$\ell=\sqrt n$ and $u=1$. Moreover,
\[
h_0(1)=h_{1/4}(1)=\frac{\sqrt n}{\sqrt{1-b}},
\]
where the value is interpreted as $+\infty$ when $b=1$, consistently
with the convention for $T_b$.
For every $s\le t_0$, condition (B1) gives
\[
\theta_{s,0}=\sqrt{2\lambda_{s,0}}
<\sqrt{n/17}<h_0(1).
\]
Thus, until the first false promotion, the proposer submits only
neutral probes and the promotion count remains zero.
A fresh probe satisfies $X_j\ge a$ if and only if all its $n$
Rademacher coordinates equal $+1$. On this event its statistic is
$\sqrt n/\sqrt{1-b}$ (or $+\infty$ when $b=1$), which exceeds the
threshold. Hence, if at least one of the $K$ probes in a round has
this property, a false promotion occurs.
Here $q=2^{-n}$ and $\pi=1-(1-q)^K$. Independence of the fresh
probe blocks and the definition of $t_0$ therefore yield
\[
\Prob(\mathfrak F_{t_0+1}^{c})
\le (1-\pi)^{t_0}
\le \exp(-\pi t_0)
\le \exp\bigl(-(1+\eta')P\bigr)
\le \exp\left\{-\frac{\eta'^2P}{2(1+\eta')}\right\}.
\]
The last inequality follows from $\eta'^2\le2(1+\eta')^2$.
Since the second term in the stated bound is nonnegative, this
proves the theorem for $n\in\{1,2\}$.

Henceforth assume $n\ge3$.
Let $\mathcal G_s$ contain all probe coordinates introduced through round $s$ and all gate randomness up to that
round. Fresh probe coordinates in round $s+1$ are independent of $\mathcal G_s$.
}

For a probe or neutral round, let
\[
X^{(s)}
=
\max_k X_{j_{s,k}}.
\]
Then
\[
\Prob(X^{(s)}\ge a)=\pi.
\]

If $\theta_{s,p}<h_0(u)$, a probe with $X^{(s)}\ge a$ crosses the threshold even though its population gain is zero.
Hence the round produces a false promotion with conditional probability at least $\pi$.

If
\[
h_0(u)
\le
\theta_{s,p}
<
h_{1/4}(u),
\]
\cref{lb:lem-calib} chooses $r_s$ so that promotion occurs exactly when
\[
X^{(s)}\ge a.
\]
The promoted candidate then has population gain
\[
r_s\tau>0,
\]
so this promotion is genuine.

Condition (B1) guarantees that, until a false promotion or the $P$-th selected probe, every round is one of
these two types. Define $I_s=\mathbf{1}\{X^{(s)}\ge a\}$ using the fresh coordinates for all $s\le t_0$;
these indicators are independent Bernoulli variables with success probability $\pi$.
On the event that no false promotion occurs and fewer than $P$ probes are selected by $t_0$, each success
$I_s=1$ must have produced a genuine selected probe. Consequently, Chernoff's inequality gives
\[
\Prob\!\left(
\mathfrak F_{t_0}^c\cap
\{\text{fewer than }P\text{ selected probes by }t_0\}
\right)
\le
\Prob\!\left(\sum_{s=1}^{t_0} I_s<P\right)
\le
\exp\left\{
-\frac{\eta'^2P}{2(1+\eta')}
\right\}.
\]

Suppose instead that $P$ genuine probe promotions have occurred. The final candidate has paired differences
\[
z
=
\tau
\sum_{j\in J}g_{\cdot j}
\]
and population gain zero. On the event
\[
\max_{|J|=P}Q_J<1+\varepsilon,
\]
we have
\[
\bar z
\ge
\tau n^{-1/2}Pa
\]
and
\[
\overline{z^2}
=
\tau^2PQ_J.
\]
Therefore
\[
T_b(z)
>
\frac{
Pa
}{
\sqrt{P(1+\varepsilon)}
}
=
a\sqrt{\frac{P}{1+\varepsilon}}.
\]
Condition (B2) makes this larger than the promotion threshold, so the final neutral candidate is promoted.

The only two ways to avoid a false promotion are therefore failure to collect $P$ probes by $t_0$ or failure of the
uniform second-moment event. Apply \cref{lb:lem-subset}.
\end{proof}

\begin{corollary}[First-order promotion-history rate]
\label{lb:cor-rate}
Let $\eta\in(0,1]$, $K,P\ge1$, and $\varsigma\ge5$ with $K\varsigma\ge e^8$. Define
\[
t=\lfloor\varsigma P\rfloor+1
\]
and
\[
\Lambda_K(\varsigma)
=
\ln(K\varsigma)
-
\frac12\ln\bigl(2\ln(K\varsigma)\bigr)
-
6.
\]
Assume the sample size satisfies
\[
\sqrt{2\Lambda_K(\varsigma)}\le\frac{\sqrt n}{8},
\qquad
\frac{(2\Lambda_K(\varsigma)+5)^2}{9n}\le\frac12,
\]
and that $\lambda_{s,p}<n/34$ for all $s\le t-1$ and $p\le\min\{P-1,s-1\}$.
If
\[
\lambda_{s,P}
\le
(1-\eta)
\Lambda_K(\varsigma)P
\]
for all $P+1\le s\le t$, then a decision-only proposer satisfies
\[
\Prob(\mathfrak F_t)
\ge
1-e^{-P/4}
-
\exp\left\{
P\ln(eK\varsigma)
-
\frac{n\eta^2}{8}
\right\}.
\]
\end{corollary}

\begin{proof}
Apply \cref{lb:thm-b} with
\[
a=\sqrt{2\Lambda_K(\varsigma)},
\qquad
\eta'=1,
\qquad
\varepsilon=\eta.
\]
The two sample-size conditions allow us to apply \cref{lb:lem-tail} and bound its correction term by $1/2$.
Since $\Lambda_K(\varsigma)\le\ln(K\varsigma)$, this gives
\[
q\ge\frac{c_R}{a}\exp\{-\Lambda_K(\varsigma)-1/2\}
\ge\frac{e^3}{2\sqrt2\,K\varsigma}
\ge\frac{6}{K\varsigma}.
\]
Hence
\[
\pi
=
1-(1-q)^K
\ge 1-e^{-6/\varsigma}
\ge
\frac{3.49}{\varsigma},
\]
which implies
\[
t_0
=
\left\lceil
\frac{2P}{\pi}
\right\rceil
\le \frac{2\varsigma P}{3.49}+1
\le
\varsigma P.
\]
The conditions of \cref{lb:thm-b} follow. Finally use
\[
\binom MP
\le
(eK\varsigma)^P
\]
and
\[
\frac{n\eta^2}{4(1+\eta)}
\ge
\frac{n\eta^2}{8}.
\]
\end{proof}

\begin{corollary}[Linear growth in the promotion count is insufficient]
\label{lb:cor-linear}
Fix $c>0$ and an integer $p_0\ge0$. Suppose
\[
\lambda_{t,p}
\le
cp
\]
for all $p\ge p_0$, and assume
\[
\sup_{t,p<p_0}\lambda_{t,p}=o(n).
\]
The latter condition is omitted when $p_0=0$.
Along any sequence satisfying
\[
n\to\infty,
\qquad
P\to\infty,
\qquad
P\ln(eK)=o(n),
\qquad
\ln K=o(\sqrt n),
\]
there exists a decision-only proposer that causes a false promotion within $O(P)$ rounds with probability tending to
one.
\end{corollary}

\begin{proof}
Choose $\varsigma$ large enough that
\[
\varsigma\ge e^8,
\qquad
\Lambda_1(\varsigma)\ge2c
\]
and apply \cref{lb:cor-rate} with $\eta=1/2$. For sufficiently large $P$,
\[
cP
\le
\frac12\Lambda_K(\varsigma)P.
\]
The pre-promotion thresholds are $o(n)$ because $P=o(n)$, and the two sample-size conditions in
\cref{lb:cor-rate} hold because $\ln K=o(\sqrt n)$. Also $P\ln(eK\varsigma)=o(n)$, so both error terms
in that corollary converge to zero.
\end{proof}

\begin{remark}[Comparison with \method{}]
\label{lb:rem-ladder}
For
\[
P+1\le t\le\varsigma P+1,
\]
the allocation used by \method{} satisfies
\[
\cL_{t,P}
\le
P\ln(eK\varsigma)
+
O\left(
\ln(K\varsigma P)
+
\ln\frac1\alpha
\right).
\]
Thus the history-dependent contribution per promotion is at most $\ln(eK\varsigma)$.
The lower bound in \cref{lb:cor-rate} has scale
\[
\Lambda_K(\varsigma)
=\ln(K\varsigma)-O\!\left(\ln\ln(K\varsigma)+1\right)
\]
per promotion. These history-dependent rates agree to first order as $K\varsigma\to\infty$.
This comparison concerns the history term, not the additional time- and error-level terms in $\cL_{t,P}$.

For example, along sequences with $P\to\infty$, $K\varsigma\to\infty$,
$\ln(K\varsigma)=o(\sqrt n)$, and $P\ln(eK\varsigma)=o(n)$, the two remainder terms in \cref{lb:cor-rate}
tend to zero for every fixed $\eta\in(0,1]$. Under its threshold conditions, the false-promotion probability
therefore tends to one.
Together with \cref{lb:cor-a}, this gives the two regimes summarized in \cref{thm:lb}.
The thresholds in these lower bounds are controlled over the stated ranges of rounds; the results do not assert
that every individual threshold must exceed the stated rate.

The lower bound applies to the threshold-gate class in \cref{lb:def-gate}. \method{}'s exact test is not literally
a member of this class, but its standardized critical value has the same first-order scale in the regime of
\cref{prop:power}. The attack uses non-binary bounded scores; the comparison does not by itself establish a lower bound
for every pass/fail promotion rule. It identifies the first-order history-dependent requirement within the stated
threshold-gate class.
\end{remark}







\section{Experimental Details}
\label{app:experiments}
\providecommand{\reuseneedspace}{\par\vskip 0pt plus 6\baselineskip\penalty-200\vskip 0pt plus -6\baselineskip\relax}
\renewcommand{\topfraction}{0.9}\renewcommand{\bottomfraction}{0.6}\renewcommand{\textfraction}{0.07}
\renewcommand{\floatpagefraction}{0.8}\setcounter{topnumber}{3}\setcounter{totalnumber}{5}
\makeatletter\setlength{\@fptop}{0pt}\setlength{\@fpsep}{24pt}\setlength{\@fpbot}{0pt plus 1fil}\makeatother

This appendix describes the protocol of the live experiments, the proposer and its prompts, the compared methods, and additional results.

\subsection{Protocol}
\label{app:experiment-protocol}

Covertype contains 581,012 examples with 54 features, and we label cover type 2 as positive.
A single random permutation partitions the data into the disjoint blocks $D_{\mathrm{tr}}$, $D_{\mathrm{dev}}$, $\mathcal E$, and $H$ of \cref{sec:exp-setup}, and the remaining examples are unused.
For each seed, $S$ is a random sample of $n$ examples from $\mathcal E$, and the $2{,}000$-example set is contained in the $10{,}000$-example set of the same seed.
All methods use the same evaluation set within a seed.
\Cref{tab:experiment-protocol} summarizes the protocol.
Every configuration is fitted once on $D_{\mathrm{tr}}$ with fitting seed 0, and its predictions are cached so that a repeated configuration reuses the same fitted model.
The experiments use NVIDIA A100 GPUs with 40 GB of memory, Python 3.11.16, scikit-learn 1.9.1, and vLLM 0.29.0.

\begin{table}[tp]
\centering
\caption{Protocol shared by all methods in the live experiments.}
\label{tab:experiment-protocol}
\small
\setlength{\tabcolsep}{6pt}
\renewcommand{\arraystretch}{1.12}
\begin{tabular}{@{}ll@{}}
\toprule
Quantity & Setting \\
\midrule
\rowcolor{black!5}\multicolumn{2}{@{}l}{\textit{Data}} \\
Training set $D_{\mathrm{tr}}$ & 50,000 examples \\
Development set $D_{\mathrm{dev}}$ & 20,000 examples \\
Evaluation pool $\mathcal E$ & 60,000 examples \\
Held-out set $H$ & 50,000 examples \\
Evaluation-set size $n$ & $2{,}000$ or $10{,}000$ \\
Seeds & 30 at each evaluation-set size, seeds 3 to 32 \\
\rowcolor{black!5}\multicolumn{2}{@{}l}{\textit{Search}} \\
Rounds $T$ & 200 \\
Candidates per round $K$ & 8 \\
Changed fields per candidate & at most 2 \\
Model families & histogram gradient boosting, logistic regression \\
Preprocessing & no, standard, or robust scaling, optional selection of 10, 20, or 30 features \\
Feature subsets & all, topographic, hillshade, soil, topographic and soil \\
Starting system $A_0$ & gradient boosting, all features, depth 3, 100 iterations, learning rate 0.1 \\
Promotion target & classification accuracy \\
\rowcolor{black!5}\multicolumn{2}{@{}l}{\textit{Proposer}} \\
Model & Qwen2.5-7B-Instruct served by vLLM \\
Sampling & temperature 0.8, top-$p$ 0.95, top-$k$ 20, repetition penalty 1.05 \\
Requests per round & at most 3, each asking for 8 configurations \\
\rowcolor{black!5}\multicolumn{2}{@{}l}{\textit{\method{}}} \\
Error levels & $\alpha=0.05$, $\alpha_c=0.05$, $\gamma=0$, $\rho=1/2$ \\
Allocation weights & $w_t=1/[t(t+1)]$ and $v_p=1/[(p+1)(p+2)]$ \\
Draft threshold & $z_{\mathrm{dev}}\ge1.645$ \\
\bottomrule
\end{tabular}
\end{table}

\subsection{Proposer and prompt templates}
\label{app:experiment-prompts}

At every round, Qwen2.5-7B-Instruct \citep{qwen25}, served through vLLM \citep{kwon2023vllm}, receives one system message and one user message.
The system message describes the task, the configuration schema, and the allowed values, and it ends with a feedback paragraph that differs between \method{} and the compared methods.
The user message contains the parent configuration, its accuracy on $D_{\mathrm{dev}}$, its error rates on 14 slices of $D_{\mathrm{dev}}$, and the previous ten rounds of search history.
Each request asks for eight complete JSON configurations.
Invalid proposals, proposals equal to the parent, and proposals that change more than two fields read by the parent's or the proposal's model family are discarded, and up to three requests are made per round.
\method{} also discards configurations already evaluated in the run.
Unfilled slots are replaced by copies of the parent, which cannot be promoted.
\Cref{tab:experiment-feedback} lists the information available to the proposer, and the templates follow.

\begin{table}[tp]
\centering
\caption{
Information in the prompt of the proposer.
Only \method{} withholds evaluation accuracies on $S$, so its proposer learns about $S$ solely through promotion decisions.
SICA-style additionally sees up to ten archive members with the highest evaluation accuracy.
}
\label{tab:experiment-feedback}
\small
\setlength{\tabcolsep}{8pt}
\renewcommand{\arraystretch}{1.12}
\begin{tabular}{@{}lcc@{}}
\toprule
Information in the prompt & Compared methods & \method{} \\
\midrule
Parent configuration                               & \yes & \yes \\
Accuracy and slice error rates on $D_{\mathrm{dev}}$ & \yes & \yes \\
Accuracies of past proposals on $D_{\mathrm{dev}}$   & \yes & \yes \\
Evaluation accuracies on $S$                       & \yes & \no  \\
Promotion decisions                                & \yes & \yes \\
Development-draft updates and submissions          & \no  & \yes \\
Accuracies on $H$                                  & \no  & \no  \\
\bottomrule
\end{tabular}
\end{table}

\reuseneedspace
\paragraph{System message.}
The following part is shared by all methods.

{\setlength{\FrameSep}{4pt}\begin{framed}
\microtypesetup{protrusion=false}\raggedright\scriptsize\ttfamily\frenchspacing\parindent=0pt\parskip=0pt\everypar{\hangindent=2em\hangafter=1}
You are an expert machine-{}learning engineer improving a scikit-{}learn pipeline. The task is binary classification on forest cover-{}type data: 54 columns (10 topographic and hillshade measurements, 4 wilderness-{}area indicators, 40 soil-{}type indicators) and 50,000 training rows. Every configuration you propose is fitted from scratch on those same rows and judged by its accuracy.\par
\mbox{}\par
A configuration is a JSON object with exactly these eleven keys:\par
-{} \textquotedbl{}model\_family\textquotedbl{}: \textquotedbl{}hist\_gbm\textquotedbl{} (histogram gradient boosting) or \textquotedbl{}logreg\textquotedbl{} (logistic regression)\par
-{} \textquotedbl{}learning\_rate\textquotedbl{}: number in [0.02, 0.3], hist\_gbm only\par
-{} \textquotedbl{}max\_depth\textquotedbl{}: integer in [2, 6], hist\_gbm only\par
-{} \textquotedbl{}max\_iter\textquotedbl{}: integer in [50, 300], boosting iterations, hist\_gbm only\par
-{} \textquotedbl{}min\_samples\_leaf\textquotedbl{}: integer in [5, 100], hist\_gbm only\par
-{} \textquotedbl{}l2\_regularization\textquotedbl{}: number in [0, 1], hist\_gbm only\par
-{} \textquotedbl{}C\textquotedbl{}: number in [0.001, 100], inverse regularization strength, logreg only\par
-{} \textquotedbl{}scaling\textquotedbl{}: \textquotedbl{}none\textquotedbl{}, \textquotedbl{}standard\textquotedbl{} or \textquotedbl{}robust\textquotedbl{}\par
-{} \textquotedbl{}feature\_selection\textquotedbl{}: \textquotedbl{}none\textquotedbl{}, \textquotedbl{}kbest\_10\textquotedbl{}, \textquotedbl{}kbest\_20\textquotedbl{} or \textquotedbl{}kbest\_30\textquotedbl{} (ANOVA F-{}test over the columns kept)\par
-{} \textquotedbl{}feature\_subset\textquotedbl{}: \textquotedbl{}all\textquotedbl{}, \textquotedbl{}topo\textquotedbl{}, \textquotedbl{}soil\textquotedbl{}, \textquotedbl{}hillshade\textquotedbl{} or \textquotedbl{}topo\_soil\textquotedbl{} -{}-{} fixed column presets: topo is elevation, aspect, slope and the three distance columns; hillshade is the three hillshade columns; soil is the 40 soil-{}type indicators; topo\_soil is both; all is every column, including the four wilderness-{}area indicators\par
-{} \textquotedbl{}seed\textquotedbl{}: 0\par
\mbox{}\par
Rules:\par
-{} Every object must carry all eleven keys. A key you leave out is not \textquotedbl{}unchanged\textquotedbl{}: it falls back to the default and silently undoes earlier work. Copy the current configuration and edit only the fields you mean to change.\par
-{} A value outside its range, an unknown key or malformed JSON discards that candidate.\par
-{} Keys that do not apply to the model family you chose are kept and ignored by the fit.\par
-{} A proposal that changes more than two fields is discarded unseen, so keep every candidate within two changes of the current configuration.\par
-{} Each proposal\textquotesingle{}s accuracy on the development set is reported back to you in the next round, with its difference from the current configuration\textquotesingle{}s. Build on the changes that raised it and abandon the ones that lowered it.\par
\mbox{}\par
Answer with one short line of reasoning and then a JSON array of configurations in a \textasciigrave{}\textasciigrave{}\textasciigrave{}json code block. If the current configuration were the starting one, a reply proposing a slower learning rate and a scaled, screened variant would look like this:\par
\mbox{}\par
\textasciigrave{}\textasciigrave{}\textasciigrave{}json\par
[\par
~~\{\textquotedbl{}model\_family\textquotedbl{}: \textquotedbl{}hist\_gbm\textquotedbl{}, \textquotedbl{}learning\_rate\textquotedbl{}: 0.05, \textquotedbl{}max\_depth\textquotedbl{}: 3, \textquotedbl{}max\_iter\textquotedbl{}: 100, \textquotedbl{}min\_samples\_leaf\textquotedbl{}: 20,\par
~~~\textquotedbl{}l2\_regularization\textquotedbl{}: 0.0, \textquotedbl{}C\textquotedbl{}: 1.0, \textquotedbl{}scaling\textquotedbl{}: \textquotedbl{}none\textquotedbl{}, \textquotedbl{}feature\_selection\textquotedbl{}: \textquotedbl{}none\textquotedbl{}, \textquotedbl{}feature\_subset\textquotedbl{}: \textquotedbl{}all\textquotedbl{},\par
~~~\textquotedbl{}seed\textquotedbl{}: 0\},\par
~~\{\textquotedbl{}model\_family\textquotedbl{}: \textquotedbl{}hist\_gbm\textquotedbl{}, \textquotedbl{}learning\_rate\textquotedbl{}: 0.1, \textquotedbl{}max\_depth\textquotedbl{}: 3, \textquotedbl{}max\_iter\textquotedbl{}: 100, \textquotedbl{}min\_samples\_leaf\textquotedbl{}: 20,\par
~~~\textquotedbl{}l2\_regularization\textquotedbl{}: 0.0, \textquotedbl{}C\textquotedbl{}: 1.0, \textquotedbl{}scaling\textquotedbl{}: \textquotedbl{}standard\textquotedbl{}, \textquotedbl{}feature\_selection\textquotedbl{}: \textquotedbl{}kbest\_20\textquotedbl{},\par
~~~\textquotedbl{}feature\_subset\textquotedbl{}: \textquotedbl{}all\textquotedbl{}, \textquotedbl{}seed\textquotedbl{}: 0\}\par
]\par
\textasciigrave{}\textasciigrave{}\textasciigrave{}\par
\end{framed}}

\subsection{Compared methods}
\label{app:experiment-baselines}

All compared methods share the proposer, the candidate space, and $D_{\mathrm{tr}}$ with \method{}, and they differ in parent selection and promotion rule.
Every run keeps an archive of distinct fitted configurations.
The incumbent is the deployed configuration, and the parent is the configuration the proposer edits in the next round.
\Cref{tab:experiment-frameworks} summarizes the methods, and the reported system of every method is the incumbent after round $T$.

\begin{table}[tp]
\centering
\caption{
Compared methods in the live experiments.
Select., Repeat., and Adapt.\ indicate whether a rule controls errors from within-round selection, repeated testing, and adaptive dependence, as in \cref{tab:live-results}.
}
\label{tab:experiment-frameworks}
\small
\setlength{\tabcolsep}{4pt}
\renewcommand{\arraystretch}{1.12}
\resizebox{\textwidth}{!}{%
\begin{tabular}{@{}lllccc@{}}
\toprule
Method & Feedback from $S$ & Promotion rule & Select. & Repeat. & Adapt. \\
\midrule
Empirical best-of-$K$ & scores         & largest positive gain on $S$      & \xmark & \xmark & \xmark \\
SICA-style            & scores         & largest positive gain on $S$      & \xmark & \xmark & \xmark \\
DGM-style             & scores         & largest positive gain on $S$      & \xmark & \xmark & \xmark \\
Elite                 & scores         & largest positive gain on $S$      & \xmark & \xmark & \xmark \\
Niche-elite           & scores         & largest positive gain on $S$      & \xmark & \xmark & \xmark \\
PACE-style            & scores         & McNemar test at level $0.05$      & \xmark & \xmark & \xmark \\
Bonferroni-style      & scores         & McNemar test at level $0.05/(TK)$ & \cmark & \cmark & \xmark \\
\addlinespace[2pt]
\rowcolor{oursbg}
\method{} (ours)      & decisions only & exact paired test with \eqref{eq:alloc} & \cmark & \cmark & \cmark \\
\bottomrule
\end{tabular}%
}
\end{table}

Empirical best-of-$K$ edits the incumbent and promotes the candidate with the largest strictly positive gain on $S$.
SICA-style follows the archive-based search of SICA \citep{sica2025}.
It also edits the incumbent, and its prompt additionally lists up to ten archive members with the highest evaluation accuracy.
DGM-style follows the parent selection of DGM \citep{dgm2025}.
Archive member $i$ is chosen as parent with probability proportional to
\[
\frac{\sigma\!\left(40(a_i-a_0-0.08)\right)}{1+e_i},
\]
where $a_i$ is its evaluation accuracy, $a_0$ is that of $A_0$, $e_i$ is the number of rounds in which it has been parent, and $\sigma$ is the logistic function.
Elite draws its parent uniformly from the four archive members with the highest evaluation accuracy, and Niche-elite groups configurations by model family and feature subset, draws one group uniformly, and uses its best member.
These two methods follow the population-based search of AlphaEvolve \citep{alphaevolve2025}.
The five methods above share the promotion rule of Empirical best-of-$K$.

PACE-style and Bonferroni-style edit the incumbent and test all $K$ candidates against it with the one-sided McNemar test \citep{dietterich1998approximate},
\[
z_{\mathrm{MC}}=\frac{n_+-n_-}{\sqrt{n_++n_-}},
\qquad
p_{\mathrm{MC}}=1-\Phi(z_{\mathrm{MC}}),
\]
with $n_+$ and $n_-$ as in \cref{sec:reuse}, $\Phi$ the standard normal distribution function, and $p_{\mathrm{MC}}=1$ when $n_++n_-=0$.
PACE-style tests every candidate at level $0.05$, and Bonferroni-style at level $0.05/(TK)=0.05/1600$.
If several candidates pass, the one with the largest gain on $S$ is promoted.

\subsection{Additional results}
\label{app:experiment-runs}

\Cref{tab:live-full} complements \cref{tab:live-results} with seed-matched differences in final population improvement and the smallest population improvement among the promotions of each method.

\begin{table}[tp]
\centering
\caption{
Additional summaries of the live experiments over 30 seeds at each evaluation-set size, in percentage points.
Diff.\ is the seed-matched difference in final population improvement between a method and \method{}, reported as mean $\pm$ standard error, so a positive value means the method ends higher.
Worst is the smallest population improvement among all promotions of a method, pooled over runs.
}
\label{tab:live-full}
\small
\setlength{\tabcolsep}{5pt}
\renewcommand{\arraystretch}{1.12}
\begin{tabular}{@{}l cc c cc@{}}
\toprule
& \multicolumn{2}{c}{Diff.\ vs \method{}} && \multicolumn{2}{c}{Worst $\uparrow$} \\
\cmidrule(lr){2-3}\cmidrule(lr){5-6}
Method & $n=2{,}000$ & $n=10{,}000$ && $n=2{,}000$ & $n=10{,}000$ \\
\midrule
Empirical best-of-$K$ & \pmse{$+0.02$}{0.25} & \pmse{$+0.00$}{0.15} && $-0.89$ & $-0.61$ \\
SICA-style & \pmse{$-0.66$}{0.29} & \pmse{$-0.63$}{0.19} && $-0.63$ & $-0.42$ \\
DGM-style & \pmse{$-3.13$}{0.31} & \pmse{$-3.16$}{0.23} && $-1.44$ & $-0.21$ \\
Elite & \pmse{$-0.30$}{0.28} & \pmse{$-0.34$}{0.18} && $-1.12$ & $-0.44$ \\
Niche-elite & \pmse{$-0.56$}{0.34} & \pmse{$-0.82$}{0.19} && $-1.13$ & $-0.41$ \\
PACE-style & \pmse{$-0.57$}{0.34} & \pmse{$+0.04$}{0.14} && $-0.34$ & $-0.18$ \\
Bonferroni-style & \pmse{$-4.67$}{0.50} & \pmse{$-0.29$}{0.17} && $0.86$ & $0.41$ \\
\addlinespace[3pt]
\rowcolor{oursbg}
\method{} (ours) & -- & -- && $1.20$ & $0.34$ \\
\bottomrule
\end{tabular}
\end{table}

\subsubsection{Robustness and Ablations}
\label{app:experiment-robustness}

The theoretical guarantee of \cref{thm:main} holds for any allocation weights $w_t$ and $v_p$ and split $\rho$ fixed before querying $S$.
We tested the sensitivity of \method{} to these choices by running 11 different allocations on 10 fresh seeds at $n=2{,}000$ (\cref{tab:app-abl}).
\method{} successfully controlled the error across all configurations and made exactly zero false promotions out of 317 total promotions.
The final population improvement remained stable and within the natural variation of the proposer.
The running certificate $C_T$ also responded naturally to the split $\rho$.

The core guarantee of \method{} is highly robust to changes in the search heuristics.
To demonstrate this reliability we evaluated an alternative variant that queries $S$ at every step.
In this setup the proposer edits the incumbent directly and the candidate with the highest development accuracy is tested on $S$ at level $\rho\delta_{t,p}$ with $K=1$.
The proposer receives only the promotion decision.
Even with this more conservative search structure the variant completely avoided false promotions across 30 seeds at both evaluation set sizes (\cref{tab:app-variant}).
It made 30 strictly positive promotions at $n=2{,}000$ and 100 at $n=10{,}000$.
These results confirm that the error control mechanism fundamentally prevents false promotions and maintains rigorous validity regardless of the specific search mechanics or allocation parameters used.

We further isolated the role of the exact test on $S$ by ablating it.
We ran \method{} at $n=10{,}000$ on 10 new seeds using a smaller development set of 500 examples and a lowered pre-screening threshold of $z_{\mathrm{dev}}>0$ to deliberately allow poor candidates to reach the evaluation set.
We compared this setup with a baseline that promotes every submitted candidate without testing it on $S$.
The proposer of the untested version is told that each submitted candidate was tested and accepted so the two versions differ only in the application of the test.
\Cref{fig:app-notest} illustrates every evaluated candidate under these conditions.
\method{} rejected all 9 candidates with population improvement at most zero and made no false promotion among its 47 accepted updates.
Without the test the process committed 6 false promotions across 5 of the 10 runs with population improvements dropping down to $-0.56$ pp.
The two tested variants achieved very similar final population improvements, with no statistically significant difference between them.
This confirms the test on $S$ is strictly necessary and effective at filtering out the false improvements that pass the initial development screen.
\\
\begin{table}[tp]
\centering
\caption{
Sensitivity of \method{} to its allocation \eqref{eq:alloc} at $n=2{,}000$ over 10 new seeds.
We change one of $\rho$ or $w_t$ or $v_p$ at a time from the default.
The variable $t$ indexes the rounds where $S$ is queried and the harmonic weight uses $c\approx0.590747$.
The columns for $S$ tests, Prom., Pop.\ Impr.\ and $C_T$ are means $\pm$ standard error over seeds.
The False / all column sums false and all promotions over the 10 runs.
Diff.\ is the seed-matched difference from the default run.
Changed counts the runs in which the default allocation would have made at least one querying or promotion decision differently.
$C_T$ re-prices the promotions of the same-seed default run at the certificate level $(1-\rho)\delta_{t,p}$.
All improvements are in percentage points.
}
\label{tab:app-abl}
\small
\setlength{\tabcolsep}{3.6pt}
\renewcommand{\arraystretch}{1.1}
\resizebox{\textwidth}{!}{%
\begin{tabular}{@{}l cc c cc c c@{}}
\toprule
 & $S$ tests & Prom. & False / & Pop.\ Impr. & Diff. & Changed & $C_T$ at default's \\
Setting & per run & per run & all & (pp) & (pp) & runs & promotions (pp) \\
\midrule
\rowcolor{oursbg}Default & \pmse{4.30}{0.37} & \pmse{2.90}{0.28} & \good{0}/29 & \pmse{6.38}{0.23} & N/A & N/A & \pmse{0.74}{0.11} \\
\addlinespace[2pt]
\multicolumn{8}{@{}l}{\emph{Split $\rho$}} \\
\quad $\rho=0.1$ & \pmse{4.00}{0.30} & \pmse{3.00}{0.15} & \good{0}/30 & \pmse{7.23}{0.36} & \pmse{$+0.85$}{0.40} & 9/10 & \pmse{0.89}{0.11} \\
\quad $\rho=0.25$ & \pmse{3.90}{0.18} & \pmse{2.60}{0.16} & \good{0}/26 & \pmse{6.52}{0.27} & \pmse{$+0.14$}{0.37} & 4/10 & \pmse{0.84}{0.11} \\
\quad $\rho=0.75$ & \pmse{4.80}{0.33} & \pmse{3.00}{0.26} & \good{0}/30 & \pmse{6.52}{0.35} & \pmse{$+0.14$}{0.29} & 1/10 & \pmse{0.59}{0.10} \\
\quad $\rho=0.9$ & \pmse{4.70}{0.33} & \pmse{3.00}{0.21} & \good{0}/30 & \pmse{6.52}{0.26} & \pmse{$+0.14$}{0.32} & 5/10 & \pmse{0.44}{0.09} \\
\addlinespace[2pt]
\multicolumn{8}{@{}l}{\emph{Submission weight $w_t$}} \\
\quad harmonic & \pmse{4.60}{0.34} & \pmse{3.10}{0.23} & \good{0}/31 & \pmse{6.94}{0.21} & \pmse{$+0.56$}{0.21} & 2/10 & \pmse{0.63}{0.10} \\
\quad flat & \pmse{3.00}{0.26} & \pmse{2.30}{0.15} & \good{0}/23 & \pmse{6.51}{0.33} & \pmse{$+0.12$}{0.43} & 10/10 & \pmse{0.31}{0.06} \\
\quad geometric & \pmse{4.30}{0.21} & \pmse{3.10}{0.23} & \good{0}/31 & \pmse{6.64}{0.21} & \pmse{$+0.26$}{0.35} & 1/10 & \pmse{0.80}{0.12} \\
\addlinespace[2pt]
\multicolumn{8}{@{}l}{\emph{Promotion weight $v_p$}} \\
\quad front-loaded & \pmse{4.90}{0.50} & \pmse{2.90}{0.18} & \good{0}/29 & \pmse{7.11}{0.37} & \pmse{$+0.73$}{0.47} & 8/10 & \pmse{0.65}{0.10} \\
\quad flat & \pmse{4.00}{0.26} & \pmse{2.80}{0.20} & \good{0}/28 & \pmse{6.85}{0.33} & \pmse{$+0.47$}{0.39} & 7/10 & \pmse{0.49}{0.09} \\
\quad geometric & \pmse{4.00}{0.30} & \pmse{3.00}{0.26} & \good{0}/30 & \pmse{6.83}{0.42} & \pmse{$+0.45$}{0.42} & 0/10 & \pmse{0.79}{0.11} \\
\bottomrule
\end{tabular}%
}
\end{table}

\begin{table}[tp]
\centering
\caption{
False promotions of the alternative \method{} variant and of the seven compared methods over the 30 seeds at each evaluation set size.
Prom.\ and False are identical to the metrics in \cref{tab:live-results}.
Runs is the number of runs out of 30 with at least one false promotion.
}
\label{tab:app-variant}
\small
\setlength{\tabcolsep}{4pt}
\renewcommand{\arraystretch}{1.12}
\begin{tabular}{@{}l ccc c ccc@{}}
\toprule
& \multicolumn{3}{c}{$n=2{,}000$} && \multicolumn{3}{c}{$n=10{,}000$} \\
\cmidrule(lr){2-4}\cmidrule(lr){6-8}
Method & Prom. & False $\downarrow$ & Runs $\downarrow$ && Prom. & False $\downarrow$ & Runs $\downarrow$ \\
\midrule
Empirical best-of-$K$ & 13.0 & 75 & 30/30 &  & 14.2 & 71 & 28/30 \\
SICA-style & 12.3 & 72 & 30/30 &  & 12.5 & 62 & 26/30 \\
DGM-style & 6.4 & 34 & 22/30 &  & 7.2 & 18 & 15/30 \\
Elite & 14.3 & 89 & 27/30 &  & 13.7 & 59 & 25/30 \\
Niche-elite & 12.3 & 71 & 28/30 &  & 12.4 & 44 & 24/30 \\
\addlinespace[3pt]
PACE-style & 6.4 & 4 & 4/30 &  & 8.1 & 3 & 3/30 \\
Bonferroni-style & 1.3 & 0 & 0/30 &  & 6.5 & 0 & 0/30 \\
\addlinespace[3pt]
\rowcolor{oursbg}Alternative variant & 1.0 & \good{0} & \good{0/30} &  & 3.3 & \good{0} & \good{0/30} \\
\bottomrule
\end{tabular}
\end{table}

\begin{figure}[tp]
\centering
\includegraphics[width=\textwidth]{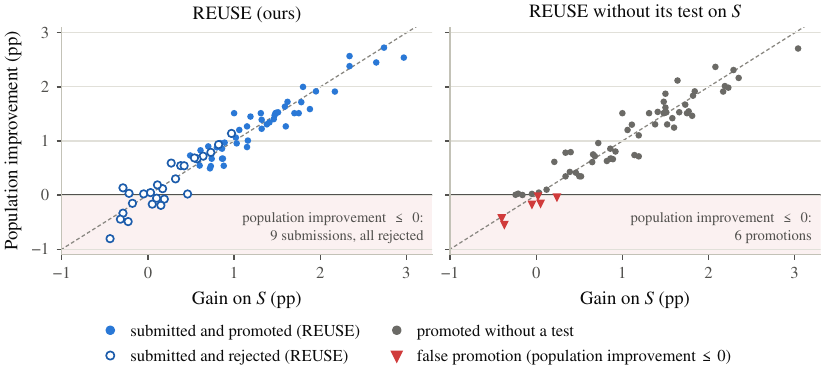}
\caption{
Performance of \method{} with and without its test on $S$ at $n=10{,}000$ over 10 seeds using a 500 example development set and pre-screening threshold $z_{\mathrm{dev}}>0$.
Each point is a submitted candidate placed at its gain on $S$ and its population improvement over the incumbent estimated on $H$.
The shaded band marks population improvement at most zero and the dashed line is equality.
The left panel shows \method{} whose test on $S$ promotes the filled points and rejects the open circles.
The right panel shows the same rule without the test which promotes every submitted candidate.
Red triangles are false promotions.
}
\label{fig:app-notest}
\end{figure}

\subsubsection{Per-seed results}
\label{app:experiment-perseed}

\Cref{tab:app-perseed-2000,tab:app-perseed-10000} list every run: for each method and seed, the final population
improvement and the numbers of false and of all promotions.

\begin{table}[!ht]
\centering
\caption{
Per-seed results at $n=2{,}000$ for the 30 seeds of \cref{tab:live-results}.
For each of the eight methods, a cell gives the final population improvement of the deployed system over $A_0$ in percentage points, estimated on the held-out set $H$, followed in parentheses by the run's false promotions and all promotions; nonzero false counts are in red.
Mean averages the final population improvement over the 30 seeds and Total sums the promotion counts.
Mean and the false counts in Total equal the corresponding columns of \cref{tab:live-results}, and the total promotions divided by 30 give its Prom.\ column.
}
\label{tab:app-perseed-2000}
\scriptsize
\setlength{\tabcolsep}{2.5pt}
\renewcommand{\arraystretch}{1.08}
\resizebox{\textwidth}{!}{%
\begin{tabular}{@{}rr@{\hspace{2pt}}lr@{\hspace{2pt}}lr@{\hspace{2pt}}lr@{\hspace{2pt}}lr@{\hspace{2pt}}lr@{\hspace{2pt}}lr@{\hspace{2pt}}l>{\columncolor{oursbg}[\tabcolsep][2pt]}r@{\hspace{2pt}}>{\columncolor{oursbg}[0pt][\tabcolsep]}l@{}}
\toprule
 & \multicolumn{2}{c}{Empirical} & \multicolumn{2}{c}{SICA-} & \multicolumn{2}{c}{DGM-} & \multicolumn{2}{c}{} & \multicolumn{2}{c}{Niche-} & \multicolumn{2}{c}{PACE-} & \multicolumn{2}{c}{Bonferroni-} & \multicolumn{2}{c}{\method{}} \\
Seed & \multicolumn{2}{c}{best-of-$K$} & \multicolumn{2}{c}{style} & \multicolumn{2}{c}{style} & \multicolumn{2}{c}{Elite} & \multicolumn{2}{c}{elite} & \multicolumn{2}{c}{style} & \multicolumn{2}{c}{style} & \multicolumn{2}{c}{(ours)} \\
\midrule
3 & 7.14 & (\textcolor{badred}{1}/9) & 5.76 & (\textcolor{badred}{3}/11) & 4.65 & (\textcolor{badred}{1}/8) & 7.14 & (\textcolor{badred}{3}/13) & 7.17 & (\textcolor{badred}{3}/20) & 7.22 & (0/9) & 0.00 & (0/0) & 7.24 & (0/3) \\
4 & 7.25 & (\textcolor{badred}{2}/12) & 7.24 & (\textcolor{badred}{3}/14) & 5.26 & (0/5) & 7.25 & (0/9) & 7.24 & (\textcolor{badred}{2}/9) & 5.75 & (0/5) & 3.21 & (0/2) & 5.73 & (0/3) \\
5 & 7.06 & (\textcolor{badred}{3}/14) & 7.41 & (\textcolor{badred}{3}/16) & 2.77 & (\textcolor{badred}{2}/7) & 4.65 & (\textcolor{badred}{4}/15) & 6.81 & (\textcolor{badred}{1}/14) & 3.79 & (0/5) & 0.00 & (0/0) & 7.26 & (0/4) \\
6 & 6.70 & (\textcolor{badred}{3}/16) & 5.55 & (\textcolor{badred}{2}/11) & 2.33 & (\textcolor{badred}{2}/4) & 6.70 & (\textcolor{badred}{5}/19) & 6.70 & (\textcolor{badred}{6}/19) & 6.72 & (\textcolor{badred}{1}/8) & 4.11 & (0/2) & 6.70 & (0/4) \\
7 & 7.14 & (\textcolor{badred}{3}/15) & 5.92 & (\textcolor{badred}{2}/14) & 3.21 & (\textcolor{badred}{1}/9) & 5.92 & (\textcolor{badred}{5}/18) & 6.50 & (\textcolor{badred}{3}/13) & 4.65 & (0/4) & 0.86 & (0/1) & 10.00 & (0/5) \\
8 & 7.24 & (\textcolor{badred}{4}/14) & 7.24 & (\textcolor{badred}{3}/13) & 2.82 & (\textcolor{badred}{1}/5) & 7.24 & (\textcolor{badred}{8}/20) & 7.24 & (\textcolor{badred}{2}/10) & 7.24 & (0/6) & 3.93 & (0/2) & 7.24 & (0/3) \\
9 & 7.02 & (\textcolor{badred}{2}/11) & 7.24 & (\textcolor{badred}{1}/9) & 4.03 & (\textcolor{badred}{2}/10) & 7.02 & (\textcolor{badred}{2}/15) & 4.65 & (\textcolor{badred}{2}/11) & 7.24 & (0/9) & 0.00 & (0/0) & 7.24 & (0/4) \\
10 & 7.10 & (\textcolor{badred}{6}/19) & 5.40 & (\textcolor{badred}{9}/19) & 4.43 & (0/6) & 5.75 & (\textcolor{badred}{4}/13) & 7.27 & (\textcolor{badred}{4}/14) & 7.17 & (0/8) & 5.76 & (0/3) & 5.06 & (0/2) \\
11 & 5.45 & (\textcolor{badred}{2}/13) & 5.75 & (\textcolor{badred}{2}/10) & 1.86 & (0/4) & 7.37 & (\textcolor{badred}{3}/12) & 7.34 & (0/10) & 6.52 & (0/6) & 1.51 & (0/1) & 5.76 & (0/3) \\
12 & 7.37 & (\textcolor{badred}{3}/16) & 6.64 & (\textcolor{badred}{2}/14) & 4.30 & (\textcolor{badred}{2}/6) & 7.37 & (\textcolor{badred}{3}/15) & 5.24 & (\textcolor{badred}{5}/14) & 7.37 & (0/7) & 0.00 & (0/0) & 10.80 & (0/3) \\
13 & 7.37 & (\textcolor{badred}{1}/9) & 5.32 & (\textcolor{badred}{3}/10) & 4.65 & (0/6) & 6.84 & (\textcolor{badred}{1}/11) & 6.15 & (0/10) & 7.24 & (0/7) & 4.65 & (0/3) & 5.92 & (0/3) \\
14 & 7.34 & (\textcolor{badred}{1}/8) & 7.34 & (\textcolor{badred}{1}/13) & 4.92 & (0/4) & 6.64 & (\textcolor{badred}{2}/14) & 5.77 & (\textcolor{badred}{2}/11) & 7.34 & (0/8) & 4.11 & (0/3) & 7.41 & (0/4) \\
15 & 7.22 & (\textcolor{badred}{3}/16) & 6.18 & (\textcolor{badred}{3}/11) & 3.89 & (\textcolor{badred}{1}/10) & 6.70 & (\textcolor{badred}{1}/10) & 5.67 & (\textcolor{badred}{2}/8) & 6.55 & (\textcolor{badred}{1}/7) & 4.65 & (0/2) & 4.80 & (0/2) \\
16 & 7.34 & (\textcolor{badred}{2}/14) & 7.34 & (\textcolor{badred}{2}/12) & 2.73 & (\textcolor{badred}{2}/8) & 7.02 & (\textcolor{badred}{2}/16) & 7.34 & (\textcolor{badred}{1}/14) & 7.14 & (0/5) & 2.73 & (0/2) & 6.40 & (0/3) \\
17 & 6.80 & (\textcolor{badred}{2}/10) & 7.37 & (\textcolor{badred}{1}/14) & 2.84 & (\textcolor{badred}{1}/7) & 7.37 & (\textcolor{badred}{1}/13) & 7.37 & (\textcolor{badred}{1}/11) & 7.37 & (0/7) & 0.00 & (0/0) & 6.31 & (0/2) \\
18 & 7.14 & (\textcolor{badred}{2}/14) & 4.65 & (\textcolor{badred}{3}/13) & 4.11 & (\textcolor{badred}{2}/6) & 5.52 & (\textcolor{badred}{4}/15) & 5.71 & (\textcolor{badred}{2}/12) & 7.24 & (0/8) & 1.20 & (0/1) & 6.84 & (0/4) \\
19 & 6.77 & (\textcolor{badred}{7}/16) & 6.84 & (\textcolor{badred}{2}/13) & 4.06 & (\textcolor{badred}{1}/6) & 5.29 & (\textcolor{badred}{6}/16) & 7.24 & (\textcolor{badred}{2}/12) & 5.14 & (\textcolor{badred}{1}/6) & 0.00 & (0/0) & 5.72 & (0/2) \\
20 & 5.87 & (\textcolor{badred}{2}/14) & 6.84 & (\textcolor{badred}{1}/10) & 4.23 & (0/4) & 5.29 & (\textcolor{badred}{2}/10) & 7.25 & (\textcolor{badred}{2}/15) & 7.14 & (0/7) & 2.19 & (0/1) & 6.45 & (0/3) \\
21 & 7.37 & (\textcolor{badred}{2}/12) & 5.75 & (\textcolor{badred}{3}/13) & 5.23 & (\textcolor{badred}{1}/8) & 7.27 & (0/12) & 5.80 & (\textcolor{badred}{5}/15) & 7.14 & (0/7) & 4.11 & (0/2) & 5.80 & (0/3) \\
22 & 7.24 & (\textcolor{badred}{1}/10) & 7.27 & (\textcolor{badred}{1}/12) & 2.77 & (0/3) & 6.18 & (\textcolor{badred}{3}/14) & 5.26 & (\textcolor{badred}{1}/9) & 6.49 & (\textcolor{badred}{1}/7) & 4.11 & (0/2) & 5.76 & (0/3) \\
23 & 7.37 & (\textcolor{badred}{1}/9) & 7.37 & (\textcolor{badred}{1}/11) & 4.56 & (\textcolor{badred}{1}/5) & 7.37 & (0/12) & 5.68 & (\textcolor{badred}{3}/11) & 5.72 & (0/5) & 4.65 & (0/2) & 9.67 & (0/4) \\
24 & 7.24 & (\textcolor{badred}{5}/19) & 5.82 & (\textcolor{badred}{2}/10) & 4.43 & (0/6) & 7.10 & (\textcolor{badred}{3}/16) & 5.76 & (\textcolor{badred}{2}/8) & 4.65 & (0/3) & 5.65 & (0/3) & 7.10 & (0/3) \\
25 & 7.34 & (\textcolor{badred}{1}/15) & 5.65 & (\textcolor{badred}{1}/11) & 3.35 & (\textcolor{badred}{2}/5) & 7.27 & (\textcolor{badred}{3}/16) & 7.34 & (\textcolor{badred}{4}/14) & 6.64 & (0/6) & 2.73 & (0/1) & 6.84 & (0/3) \\
26 & 6.80 & (\textcolor{badred}{5}/15) & 5.52 & (\textcolor{badred}{3}/9) & 4.65 & (\textcolor{badred}{2}/9) & 7.12 & (\textcolor{badred}{2}/12) & 7.05 & (\textcolor{badred}{4}/16) & 6.08 & (0/5) & 3.21 & (0/2) & 6.45 & (0/3) \\
27 & 6.70 & (\textcolor{badred}{3}/12) & 7.14 & (\textcolor{badred}{2}/12) & 4.30 & (\textcolor{badred}{1}/6) & 6.70 & (\textcolor{badred}{5}/17) & 7.14 & (\textcolor{badred}{3}/14) & 7.24 & (0/6) & 2.95 & (0/1) & 7.02 & (0/3) \\
28 & 6.95 & (\textcolor{badred}{1}/11) & 3.49 & (\textcolor{badred}{4}/11) & 3.51 & (\textcolor{badred}{1}/8) & 7.14 & (\textcolor{badred}{2}/11) & 6.95 & (\textcolor{badred}{2}/14) & 7.12 & (0/8) & 0.00 & (0/0) & 7.18 & (0/4) \\
29 & 7.14 & (\textcolor{badred}{4}/12) & 6.70 & (\textcolor{badred}{3}/15) & 2.92 & (\textcolor{badred}{1}/4) & 6.81 & (\textcolor{badred}{3}/13) & 6.45 & (\textcolor{badred}{2}/8) & 7.24 & (0/6) & 4.07 & (0/2) & 10.00 & (0/5) \\
30 & 7.34 & (\textcolor{badred}{1}/10) & 7.05 & (\textcolor{badred}{1}/14) & 4.65 & (\textcolor{badred}{4}/8) & 7.06 & (\textcolor{badred}{5}/21) & 5.00 & (\textcolor{badred}{1}/10) & 6.50 & (0/7) & 0.00 & (0/0) & 7.34 & (0/3) \\
31 & 7.24 & (\textcolor{badred}{1}/13) & 7.05 & (\textcolor{badred}{3}/12) & 4.11 & (\textcolor{badred}{2}/8) & 7.24 & (\textcolor{badred}{3}/15) & 7.05 & (\textcolor{badred}{3}/16) & 5.68 & (0/7) & 0.00 & (0/0) & 7.10 & (0/3) \\
32 & 7.25 & (\textcolor{badred}{1}/12) & 5.94 & (\textcolor{badred}{2}/13) & 5.14 & (\textcolor{badred}{1}/8) & 7.14 & (\textcolor{badred}{4}/16) & 5.71 & (\textcolor{badred}{1}/8) & 4.11 & (0/4) & 0.00 & (0/0) & 7.41 & (0/3) \\
\midrule
\multicolumn{1}{@{}l}{Mean}
& \multicolumn{2}{c}{7.04}
& \multicolumn{2}{c}{6.36}
& \multicolumn{2}{c}{3.89}
& \multicolumn{2}{c}{6.72}
& \multicolumn{2}{c}{6.46}
& \multicolumn{2}{c}{6.45}
& \multicolumn{2}{c}{2.35}
& \multicolumn{2}{c}{7.02} \\

\multicolumn{1}{@{}l}{Total}
& \multicolumn{2}{c}{(\textcolor{badred}{75}/390)}
& \multicolumn{2}{c}{(\textcolor{badred}{72}/370)}
& \multicolumn{2}{c}{(\textcolor{badred}{34}/193)}
& \multicolumn{2}{c}{(\textcolor{badred}{89}/429)}
& \multicolumn{2}{c}{(\textcolor{badred}{71}/370)}
& \multicolumn{2}{c}{(\textcolor{badred}{4}/193)}
& \multicolumn{2}{c}{(0/38)}
& \multicolumn{2}{c}{(0/97)} \\
\bottomrule
\end{tabular}%
}
\end{table}

\begin{table}[!ht]
\centering
\caption{
Per-seed results at $n=10{,}000$ for the 30 seeds of \cref{tab:live-results}.
For each of the eight methods, a cell gives the final population improvement of the deployed system over $A_0$ in percentage points, estimated on the held-out set $H$, followed in parentheses by the run's false promotions and all promotions; nonzero false counts are in red.
Mean averages the final population improvement over the 30 seeds and Total sums the promotion counts.
Mean and the false counts in Total equal the corresponding columns of \cref{tab:live-results}, and the total promotions divided by 30 give its Prom.\ column.
}
\label{tab:app-perseed-10000}
\scriptsize
\setlength{\tabcolsep}{2.5pt}
\renewcommand{\arraystretch}{1.08}
\resizebox{\textwidth}{!}{%
\begin{tabular}{@{}rr@{\hspace{2pt}}lr@{\hspace{2pt}}lr@{\hspace{2pt}}lr@{\hspace{2pt}}lr@{\hspace{2pt}}lr@{\hspace{2pt}}lr@{\hspace{2pt}}l>{\columncolor{oursbg}[\tabcolsep][2pt]}r@{\hspace{2pt}}>{\columncolor{oursbg}[0pt][\tabcolsep]}l@{}}
\toprule
 & \multicolumn{2}{c}{Empirical} & \multicolumn{2}{c}{SICA-} & \multicolumn{2}{c}{DGM-} & \multicolumn{2}{c}{} & \multicolumn{2}{c}{Niche-} & \multicolumn{2}{c}{PACE-} & \multicolumn{2}{c}{Bonferroni-} & \multicolumn{2}{c}{\method{}} \\
Seed & \multicolumn{2}{c}{best-of-$K$} & \multicolumn{2}{c}{style} & \multicolumn{2}{c}{style} & \multicolumn{2}{c}{Elite} & \multicolumn{2}{c}{elite} & \multicolumn{2}{c}{style} & \multicolumn{2}{c}{style} & \multicolumn{2}{c}{(ours)} \\
\midrule
3 & 7.17 & (\textcolor{badred}{1}/17) & 4.65 & (0/5) & 2.92 & (\textcolor{badred}{1}/7) & 7.41 & (\textcolor{badred}{4}/20) & 7.24 & (\textcolor{badred}{2}/10) & 7.24 & (0/9) & 7.10 & (0/7) & 6.70 & (0/5) \\
4 & 7.37 & (\textcolor{badred}{1}/13) & 7.44 & (\textcolor{badred}{2}/14) & 5.27 & (\textcolor{badred}{1}/8) & 7.27 & (\textcolor{badred}{3}/16) & 5.61 & (0/7) & 7.37 & (0/7) & 7.24 & (0/7) & 7.25 & (0/8) \\
5 & 6.64 & (\textcolor{badred}{2}/14) & 7.37 & (\textcolor{badred}{2}/13) & 2.19 & (0/5) & 6.45 & (\textcolor{badred}{2}/8) & 5.80 & (0/9) & 7.24 & (0/9) & 7.10 & (0/7) & 7.37 & (0/6) \\
6 & 7.41 & (\textcolor{badred}{2}/15) & 7.34 & (\textcolor{badred}{2}/13) & 5.46 & (0/7) & 7.24 & (\textcolor{badred}{1}/9) & 6.57 & (\textcolor{badred}{2}/20) & 7.27 & (0/10) & 4.65 & (0/4) & 7.34 & (0/7) \\
7 & 7.41 & (\textcolor{badred}{3}/17) & 5.92 & (\textcolor{badred}{1}/13) & 4.65 & (\textcolor{badred}{1}/8) & 6.62 & (\textcolor{badred}{4}/21) & 5.18 & (0/13) & 7.24 & (0/9) & 7.05 & (0/7) & 7.24 & (0/6) \\
8 & 7.34 & (\textcolor{badred}{3}/13) & 7.34 & (\textcolor{badred}{2}/13) & 2.55 & (0/4) & 7.34 & (\textcolor{badred}{1}/15) & 5.46 & (\textcolor{badred}{1}/12) & 7.22 & (0/8) & 6.45 & (0/6) & 6.62 & (0/6) \\
9 & 5.75 & (\textcolor{badred}{2}/11) & 7.24 & (\textcolor{badred}{5}/15) & 4.65 & (\textcolor{badred}{1}/8) & 7.27 & (\textcolor{badred}{1}/13) & 7.37 & (\textcolor{badred}{2}/15) & 7.17 & (0/9) & 7.24 & (0/7) & 7.24 & (0/6) \\
10 & 7.26 & (0/10) & 7.14 & (\textcolor{badred}{3}/15) & 4.71 & (0/11) & 7.24 & (\textcolor{badred}{2}/21) & 6.02 & (0/12) & 7.24 & (0/8) & 6.45 & (0/7) & 7.34 & (0/6) \\
11 & 7.41 & (\textcolor{badred}{2}/16) & 4.46 & (\textcolor{badred}{3}/9) & 2.73 & (\textcolor{badred}{1}/9) & 5.76 & (0/9) & 7.34 & (\textcolor{badred}{3}/18) & 7.05 & (\textcolor{badred}{1}/9) & 6.52 & (0/6) & 6.45 & (0/6) \\
12 & 7.34 & (\textcolor{badred}{3}/14) & 7.37 & (\textcolor{badred}{1}/17) & 3.27 & (0/5) & 6.52 & (\textcolor{badred}{2}/12) & 6.24 & (\textcolor{badred}{2}/11) & 7.10 & (0/8) & 7.24 & (0/7) & 7.10 & (0/5) \\
13 & 7.34 & (\textcolor{badred}{2}/13) & 6.41 & (\textcolor{badred}{1}/10) & 4.65 & (0/6) & 7.34 & (\textcolor{badred}{2}/15) & 4.60 & (\textcolor{badred}{1}/8) & 7.34 & (0/8) & 7.24 & (0/7) & 7.34 & (0/6) \\
14 & 7.12 & (\textcolor{badred}{2}/12) & 7.14 & (\textcolor{badred}{4}/15) & 4.19 & (\textcolor{badred}{2}/9) & 5.52 & (\textcolor{badred}{3}/11) & 7.37 & (\textcolor{badred}{4}/14) & 7.24 & (\textcolor{badred}{1}/8) & 7.24 & (0/7) & 7.24 & (0/5) \\
15 & 7.37 & (\textcolor{badred}{1}/13) & 6.60 & (\textcolor{badred}{3}/15) & 4.52 & (\textcolor{badred}{1}/9) & 7.37 & (\textcolor{badred}{1}/14) & 5.65 & (\textcolor{badred}{2}/11) & 7.41 & (0/11) & 5.65 & (0/6) & 6.70 & (0/5) \\
16 & 7.24 & (\textcolor{badred}{1}/8) & 7.34 & (\textcolor{badred}{1}/16) & 5.24 & (0/7) & 7.10 & (\textcolor{badred}{3}/16) & 5.72 & (0/7) & 7.24 & (0/6) & 7.24 & (0/5) & 7.05 & (0/6) \\
17 & 7.12 & (\textcolor{badred}{2}/11) & 6.28 & (0/10) & 2.54 & (\textcolor{badred}{1}/6) & 7.41 & (\textcolor{badred}{1}/17) & 7.37 & (\textcolor{badred}{1}/12) & 7.05 & (0/7) & 5.65 & (0/5) & 7.26 & (0/7) \\
18 & 7.24 & (\textcolor{badred}{1}/10) & 5.00 & (\textcolor{badred}{3}/11) & 4.79 & (0/6) & 7.24 & (0/14) & 5.94 & (\textcolor{badred}{1}/9) & 7.10 & (0/7) & 7.37 & (0/7) & 7.10 & (0/6) \\
19 & 7.24 & (\textcolor{badred}{4}/18) & 7.24 & (\textcolor{badred}{4}/18) & 3.25 & (\textcolor{badred}{1}/8) & 7.27 & (0/9) & 7.24 & (\textcolor{badred}{2}/16) & 7.24 & (0/10) & 7.37 & (0/7) & 7.05 & (0/6) \\
20 & 7.37 & (\textcolor{badred}{4}/18) & 7.34 & (\textcolor{badred}{2}/16) & 5.27 & (\textcolor{badred}{1}/10) & 7.37 & (\textcolor{badred}{2}/16) & 5.92 & (\textcolor{badred}{2}/13) & 7.24 & (0/7) & 7.24 & (0/7) & 6.48 & (0/6) \\
21 & 7.24 & (\textcolor{badred}{3}/16) & 7.24 & (\textcolor{badred}{2}/15) & 2.92 & (0/5) & 6.01 & (\textcolor{badred}{3}/13) & 7.24 & (\textcolor{badred}{1}/15) & 7.14 & (0/8) & 7.05 & (0/7) & 10.00 & (0/6) \\
22 & 7.24 & (\textcolor{badred}{3}/13) & 5.92 & (\textcolor{badred}{2}/12) & 4.65 & (0/7) & 5.72 & (\textcolor{badred}{1}/12) & 5.76 & (\textcolor{badred}{1}/10) & 7.34 & (0/8) & 7.34 & (0/7) & 6.40 & (0/6) \\
23 & 7.12 & (\textcolor{badred}{6}/19) & 6.02 & (0/6) & 4.24 & (\textcolor{badred}{3}/10) & 5.87 & (\textcolor{badred}{2}/13) & 6.84 & (\textcolor{badred}{3}/18) & 7.34 & (0/8) & 7.24 & (0/6) & 6.84 & (0/6) \\
24 & 7.41 & (\textcolor{badred}{5}/18) & 5.75 & (\textcolor{badred}{3}/13) & 2.73 & (0/5) & 7.24 & (\textcolor{badred}{3}/14) & 5.92 & (\textcolor{badred}{1}/14) & 7.24 & (0/7) & 7.02 & (0/7) & 7.10 & (0/6) \\
25 & 7.24 & (\textcolor{badred}{1}/17) & 5.72 & (\textcolor{badred}{1}/10) & 4.52 & (0/10) & 7.34 & (0/8) & 5.94 & (\textcolor{badred}{1}/13) & 7.02 & (0/8) & 7.24 & (0/7) & 9.08 & (0/6) \\
26 & 7.34 & (\textcolor{badred}{3}/16) & 7.24 & (\textcolor{badred}{4}/16) & 4.70 & (\textcolor{badred}{1}/8) & 7.24 & (\textcolor{badred}{3}/17) & 7.24 & (\textcolor{badred}{3}/13) & 7.41 & (\textcolor{badred}{1}/9) & 6.60 & (0/6) & 6.84 & (0/4) \\
27 & 7.10 & (\textcolor{badred}{2}/10) & 6.40 & (\textcolor{badred}{1}/7) & 4.43 & (0/7) & 7.14 & (\textcolor{badred}{3}/12) & 7.24 & (\textcolor{badred}{2}/10) & 7.24 & (0/6) & 7.37 & (0/6) & 7.26 & (0/6) \\
28 & 7.25 & (\textcolor{badred}{2}/12) & 6.70 & (\textcolor{badred}{2}/12) & 4.65 & (\textcolor{badred}{1}/9) & 5.58 & (\textcolor{badred}{2}/11) & 7.25 & (0/10) & 7.24 & (0/7) & 7.24 & (0/7) & 6.45 & (0/6) \\
29 & 7.14 & (\textcolor{badred}{6}/21) & 7.24 & (\textcolor{badred}{2}/9) & 3.21 & (0/3) & 7.24 & (\textcolor{badred}{2}/11) & 6.24 & (\textcolor{badred}{2}/14) & 7.34 & (0/8) & 7.37 & (0/6) & 7.37 & (0/6) \\
30 & 7.28 & (\textcolor{badred}{3}/15) & 5.76 & (\textcolor{badred}{3}/11) & 4.52 & (0/5) & 7.37 & (0/13) & 7.28 & (\textcolor{badred}{1}/12) & 7.10 & (0/8) & 5.82 & (0/5) & 7.22 & (0/5) \\
31 & 7.12 & (\textcolor{badred}{1}/15) & 5.94 & (0/11) & 2.77 & (\textcolor{badred}{1}/5) & 7.37 & (\textcolor{badred}{5}/19) & 5.18 & (\textcolor{badred}{3}/13) & 7.26 & (0/8) & 7.37 & (0/7) & 7.17 & (0/7) \\
32 & 7.25 & (0/11) & 7.37 & (\textcolor{badred}{3}/16) & 4.65 & (\textcolor{badred}{1}/10) & 5.71 & (\textcolor{badred}{3}/11) & 6.40 & (\textcolor{badred}{1}/14) & 7.14 & (0/8) & 7.24 & (0/7) & 7.14 & (0/5) \\
\midrule
\multicolumn{1}{@{}l}{Mean}
& \multicolumn{2}{c}{7.19}
& \multicolumn{2}{c}{6.56}
& \multicolumn{2}{c}{4.03}
& \multicolumn{2}{c}{6.85}
& \multicolumn{2}{c}{6.37}
& \multicolumn{2}{c}{7.23}
& \multicolumn{2}{c}{6.90}
& \multicolumn{2}{c}{7.19} \\

\multicolumn{1}{@{}l}{Total}
& \multicolumn{2}{c}{(\textcolor{badred}{71}/426)}
& \multicolumn{2}{c}{(\textcolor{badred}{62}/376)}
& \multicolumn{2}{c}{(\textcolor{badred}{18}/217)}
& \multicolumn{2}{c}{(\textcolor{badred}{59}/410)}
& \multicolumn{2}{c}{(\textcolor{badred}{44}/373)}
& \multicolumn{2}{c}{(\textcolor{badred}{3}/243)}
& \multicolumn{2}{c}{(0/194)}
& \multicolumn{2}{c}{(0/177)} \\
\bottomrule
\end{tabular}%
}
\end{table}

\clearpage
\subsubsection{Trajectories}
\label{app:experiment-trajectories}

\Cref{fig:app-trajectory-se} separates the trajectories of \cref{fig:trajectory} by method and adds standard
errors across seeds.

\begin{figure}[!htbp]
\centering
\includegraphics[width=\textwidth]{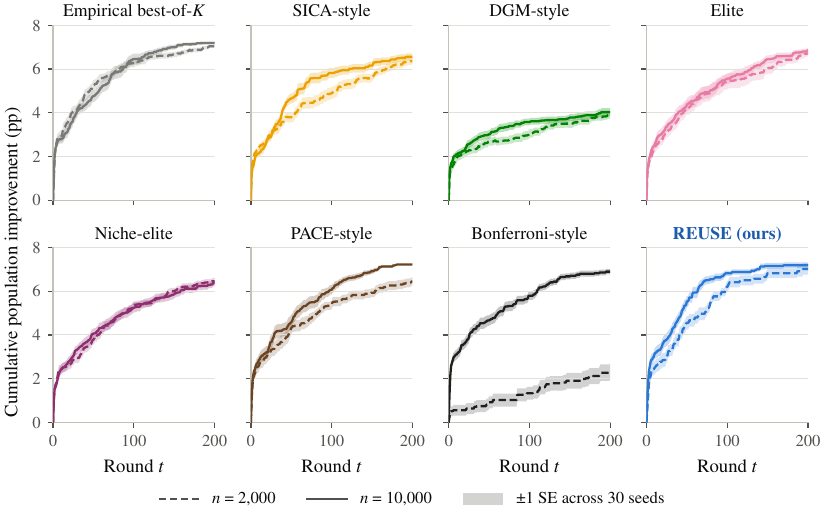}
\caption{
Mean population improvement of the deployed system over $A_0$ by round, estimated on $H$ and averaged over 30 seeds, with one panel per method; dashed lines use $n=2{,}000$ and solid lines use $n=10{,}000$.
Shaded bands show $\pm 1$ standard error across the 30 seeds.
}
\label{fig:app-trajectory-se}
\end{figure}

\subsubsection{Promotions and final performance}
\label{app:experiment-promotions}

\Cref{fig:app-gain-vs-promotions} plots the final population improvement of every run against its number of promotions.
\method{} reaches a final population improvement comparable to that of Empirical best-of-$K$ (\cref{tab:live-full})
with far fewer promotions, a median of 3 against 13 per run at $n=2{,}000$ and 6 against 14 at $n=10{,}000$.
No run of the seven compared methods improves on $A_0$ by more than 7.44~pp, whereas four runs of \method{} at
$n=2{,}000$ and two at $n=10{,}000$ improve on it by 9.08 to 10.80~pp.

\begin{figure}[!htbp]
\centering
\includegraphics[width=\textwidth]{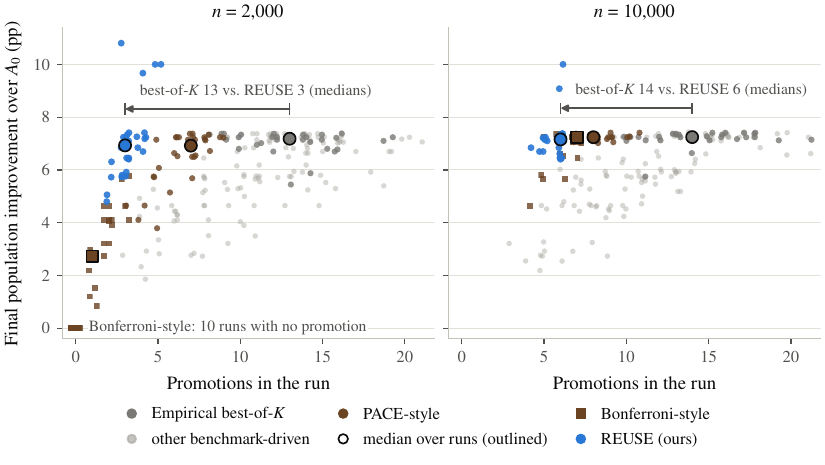}
\caption{
Final population improvement over $A_0$ against the number of promotions in the run, estimated on $H$, one point per
run with small horizontal jitter, for the 30 runs of each method at each evaluation-set size.
Light gray points are the other benchmark-driven frameworks (SICA-style, DGM-style, Elite and Niche-elite), and
outlined markers place Empirical best-of-$K$, PACE-style, Bonferroni-style and \method{} at their median number of
promotions and median final population improvement.
The arrow joins the medians of Empirical best-of-$K$ and \method{}.
}
\label{fig:app-gain-vs-promotions}
\end{figure}

\subsubsection{Certificates}
\label{app:experiment-certificates}

\method{} records the running certificate $C_T$ of \cref{sec:reuse} and the direct certificate $D_T$ of \cref{app:direct}.
Both are lower confidence bounds computed from the reused evaluation set alone, so they are reported separately
from the population-gain estimates on $H$.
\Cref{tab:reuse-certificates} compares them with the realized population improvement, and none of the 274
per-promotion lower bounds $\ell$ that enter $C_T$ exceeds the population improvement of its promotion.
The direct certificate is larger than the running certificate in all 60 runs, because it pays one confidence
margin for the whole trajectory rather than one per promotion.

\begin{table}[tp]
\centering
\caption{
Certificates of cumulative improvement for \method{}, over the 30 runs at each evaluation-set size.
The final population improvement over $A_0$ is the mean over runs, estimated on $H$ as in \cref{tab:live-results}.
$C_T$ and $D_T$ are reported as the mean over runs with the range across runs in brackets.
Valid gives, for $C_T$ and for $D_T$, the number of runs (of 30) in which the certificate is at most the
realized final population improvement.
The final population improvement, $C_T$ and $D_T$ are in percentage points.
}
\label{tab:reuse-certificates}
\small
\setlength{\tabcolsep}{5pt}
\renewcommand{\arraystretch}{1.12}
\begin{tabular}{@{}cccccc@{}}
\toprule
$n$
& \shortstack{Final population\\improvement}
& \shortstack{Promotions\\per run}
& \shortstack{Running\\certificate $C_T$}
& \shortstack{Direct\\certificate $D_T$}
& \shortstack{Valid\\$C_T$ / $D_T$} \\
\midrule
$2{,}000$ & 7.02 & 3.2 & 0.92 {\scriptsize[0.13, 1.76]} & 3.10 {\scriptsize[1.16, 5.49]} & 30 / 30 \\
$10{,}000$ & 7.19 & 5.9 & 2.67 {\scriptsize[2.06, 3.88]} & 5.33 {\scriptsize[4.34, 6.92]} & 30 / 30 \\
\bottomrule
\end{tabular}
\end{table}

\end{document}